\documentclass{article}

\usepackage[margin=1in]{geometry}
\usepackage{microtype}
\usepackage{graphicx}
\usepackage{subcaption}
\usepackage{booktabs}
\usepackage{multirow}
\usepackage{enumitem}
\usepackage{hyperref}
\usepackage{makecell}
\usepackage{natbib}
\usepackage{textcomp}

\usepackage{amsmath}
\usepackage{amssymb}
\usepackage{amsthm}
\usepackage[ruled]{algorithm2e}
\usepackage{graphicx}
\usepackage{xcolor}
\usepackage{wrapfig}
\usepackage[thinlines]{easytable}
\usepackage{titlesec}
\usepackage{subcaption}
\usepackage{mathtools}
\usepackage{bold-extra}
\usepackage{bm}
\usepackage{booktabs}
\usepackage{cancel}
\usepackage[font=footnotesize]{caption}
\usepackage{verbatim}

\newtheorem{theorem}{Theorem}
\newtheorem{lemma}{Lemma}

\newtheorem{proposition}{Proposition}

\newtheorem{assumption}{Assumption}

\newcommand{\norm}[1]{\left|\left|#1\right|\right|}
\DeclareMathOperator*{\argmax}{argmax}
\DeclareMathOperator*{\argmin}{argmin}

\newcommand{\E}[2]{\mathbb{E}_{#1}\left[#2\right]}

\newcommand{\ind}[1]{\mathbf{1}\left\{#1\right\}}

\newcommand{\supp}{\text{supp}}

\newcommand{\loss}[0]{\ell_\theta}
\newcommand{\gtilde}[0]{\widetilde{g}}
\newcommand{\indep}[0]{\perp \!\!\! \perp}

\newcommand{\dhat}[0]{\hat{\delta}}
\newcommand{\hatfkliep}[0]{\hat{f}_{\mathrm{KLIEP}}}

\newcommand{\Enoisy}[2]{\mathbb{E}^{\sigma}_{#1}\left[#2\right]}
\newcommand{\gbar}[0]{\bar{g}}
\newcommand{\Lgbar}[0]{\mathcal{L}_{\bar{g}}}

\newcommand{\dbar}[0]{\bar{\delta}}
\newcommand{\Mhat}[0]{\hat{M}}
\newcommand{\I}[2]{\mathlarger{\int_{#1}^{#2}}}
\newcommand{\lt}{\left(}
\newcommand{\rt}{\right)}
\newcommand{\ra}{\rightarrow}
\newcommand{\defeq}{:=}

\usepackage{xspace}
\newcommand{\cbiw}{{\sc CBIW}\xspace}
\newcommand{\iw}{{\sc IW}\xspace}
\newcommand{\mandoline}{{\sc Mandoline}\xspace}

\newcommand{\para}[1]{\noindent{\bf #1}}

\newcommand{\X}[0]{\mathcal{X}}
\newcommand{\p}[0]{\mathcal{P}}

\newcommand{\Y}[0]{\mathcal{Y}}

\newcommand{\D}[0]{\mathcal{D}}
\newcommand{\N}[0]{\mathcal{N}}

\newcommand{\R}[0]{\mathbb{R}}
\newcommand{\dx}[0]{\mathrm{d} x}
\newcommand{\dy}[0]{\mathrm{d} y}
\newcommand{\Lt}[0]{\mathcal{L}_t}
\newcommand{\Lg}[0]{\mathcal{L}_g}
\newcommand{\Lghat}[0]{\hat{\mathcal{L}}_g}

\newif\ifsinglecolumn

\newcommand{\name}{\mandoline}

\title{\name: Model Evaluation under Distribution Shift}

\begin{document}

\author{%
  Mayee Chen*, Karan Goel*, Nimit Sohoni*, \\
  Fait Poms, Kayvon Fatahalian, Christopher R\'e \\
  \emph{Stanford University} \\
  \small mfchen@stanford.edu, kgoel@cs.stanford.edu, nims@stanford.edu, \\
  \small fpoms@cs.stanford.edu, kayvonf@cs.stanford.edu, chrismre@cs.stanford.edu
}

\maketitle

\begin{abstract}
Machine learning models are often deployed in different settings than they were trained and validated on, posing a challenge to practitioners who wish to predict how well the deployed model will perform on a target distribution.
If an unlabeled sample from the target distribution is available, along with a labeled sample from a possibly different source distribution, standard approaches such as importance weighting can be applied to estimate performance on the target.
However, importance weighting struggles when the source and target distributions have non-overlapping support or are high-dimensional.
Taking inspiration from fields such as epidemiology and polling, we develop \name, a new evaluation framework that mitigates these issues. 
Our key insight is that practitioners may have prior knowledge about the ways in which the distribution shifts, which we can use to better guide the importance weighting procedure.
Specifically, users write simple ``slicing functions''---noisy, potentially correlated binary functions intended to capture possible axes of distribution shift---to compute reweighted performance estimates.
We further describe a density ratio estimation framework for the slices and show how its estimation error scales with slice quality and dataset size.
Empirical validation on NLP and vision tasks shows that \name can estimate performance on the target distribution up to $3\times$ more accurately compared to standard baselines.
\end{abstract}

{\let\thefootnote\relax \footnotetext{*Equal contribution.} }

\setlength{\parskip}{0.25em}

\section{Introduction}
Model evaluation is central to the machine learning (ML) pipeline. 
Ostensibly, the goal of evaluation is for practitioners to determine if their models will perform well when deployed. 
Unfortunately, standard evaluation falls short of this goal on two counts. 
First, evaluation data is frequently from a different distribution than the one on which the model will be deployed, for instance due to data collection procedures or distributional shifts over time.
Second, practitioners play a passive role in evaluation, which misses an opportunity to leverage their understanding of what distributional shifts they expect and what shifts they want their model to be robust to. 

By contrast, fields such as polling~\citep{isakov2020towards} and epidemiology~\citep{austin2011an} ``adjust'' evaluation estimates to account for such shifts using techniques such as propensity weighting, correcting for differences between treatment and control groups in observational studies~\citep{rosenbaum1983central, d1998propensity}.

\begin{figure*}[t!]
    \centering
    \includegraphics[width=\linewidth]{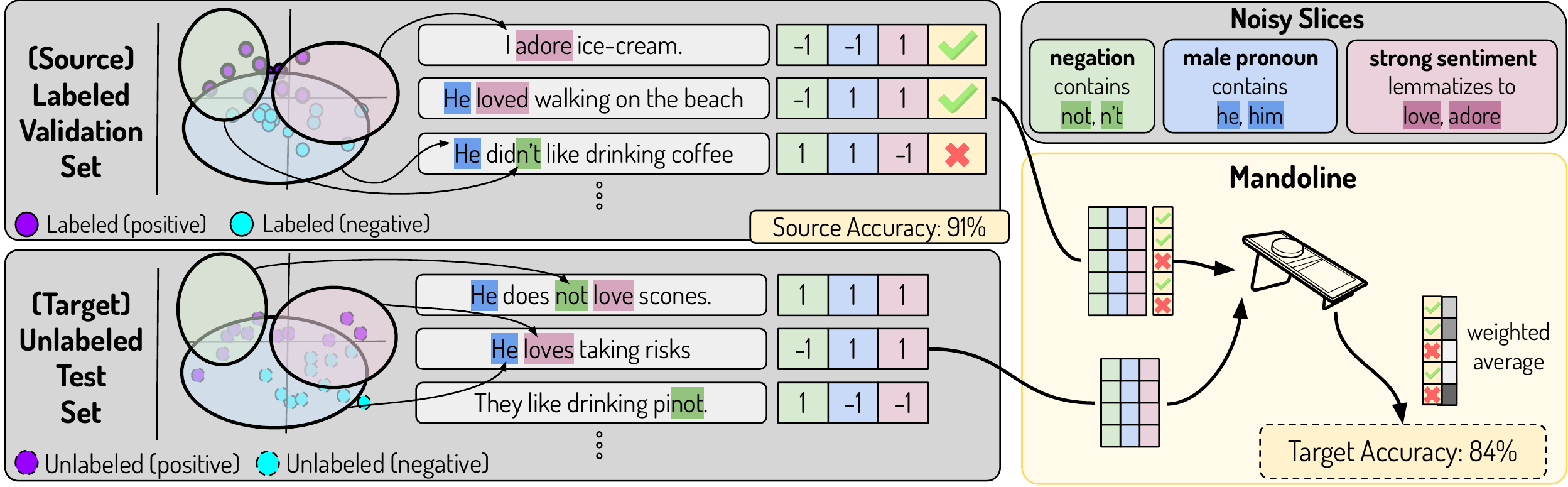}
    \caption[\name\ Schematic]{
    Schematic outlining the \name workflow for a sentiment analysis task. \textit{(left)} Given labeled source and unlabeled target data, users write noisy, correlated slicing functions (indicated by color) to create a slice representation for the datasets. \textit{(right)} \name uses these slices to perform density ratio estimation. Then, it uses these ratios to output a performance estimate for the target dataset by reweighting the source data.}
    \label{fig:splash}
\end{figure*}

Taking inspiration from this, we develop \name (Figure~\ref{fig:splash}), a user-guided, theoretically grounded framework for evaluating ML models.
Using a labeled validation set from a source distribution and an \emph{unlabeled} set from a target (test) distribution of interest, \name computes performance estimates for the target distribution using adjusted estimates on the validation set.
%


A central challenge is how to account for the shift between source and target distributions when adjusting estimates. 
%
%
A straightforward approach is to use importance weighting (\iw)---estimating the density ratio between source and target data to adjust performance.
\iw works well when the source and target distribution overlap significantly, but performs poorly when the distributions' supports have little overlap, as is common under distribution shift. 
%
%
%
Additionally, \iw works well in low dimensions but struggles with large variances of estimated ratios in high-dimensional settings~\citep{pmlr-v89-stojanov19a}.

%

Our key insight is that practitioners can use their understanding to identify the axes along which the distributions may have changed. 
%
Practitioners commonly express this knowledge programmatically by grouping (``slicing'') data along such axes for analysis. 
%
For example, in sentiment analysis a heuristic may use the word ``not'' to detect sentence negation  (Figure~\ref{fig:splash}). 
Or, slices can identify critical data subsets (e.g. X-rays of critically ill patients with chest drains~\citep{oakden2019hidden} or demographic slices when detecting toxic comments~\citep{borkan2019nuanced}).

\name leverages precisely this information:
 users construct ``slicing functions'' on the data---either programmatically mapping examples to slices, or using metadata associated with examples to group them.
These slices create a common representation in which to project and compare source and target data, reducing dimensionality and mitigating support shift. 
%
Compared to standard \iw, \name simplifies density ratio estimation by relying on source and target slice statistics, rather than raw features. 

When the slices accurately capture the relevant distribution shift, \name can be instantiated with many standard density ratio estimation methods on the slices. 
%
However, since practitioners rely on heuristics to write slicing functions, slices can be noisy, correlated, and incomplete, and direct density ratio estimation does not handle this possible misspecification.
We thus represent the distribution of slices as a graphical model to incorporate knowledge of their correlations and incompleteness, and provide a novel extension of the LL-KLIEP method~\citep{Tsuboi2009} that can denoise the density ratio based on the practitioner's confidence and prior knowledge of slice quality. These density ratios are then used to generate importance weights for evaluating models.

Theoretically, we provide a bias-variance decomposition of the error of our reweighting approach. The bias depends on how well the user-specified slices capture distribution shift, and the variance depends on distributional properties and the amount of data.
Compared to standard \iw, which lacks theoretical guarantees in the presence of support shift,
\name~can still obtain error bounds given the assumptions in Section~\ref{sec:alg} (which do not preclude support shift).


Empirically, we verify that \name outperforms standard baselines for density ratio estimation on both synthetic and real-world tasks from natural language processing and computer vision. 
When slices completely capture shifts without noise, \name reduces estimation error by up to $3\times$ compared to standard \iw baselines. Even with noisy slices, \name exhibits little performance degradation. When slices are underspecified and noisy and do not completely capture large distributional shifts, \name continues to outperform baselines by up to $2\times$. 

We conclude with a discussion of slice design. We explore an extremely challenging underspecified setting, wherein we highlight how clever slice design can reduce estimation error by $2.86\times$, and show that \name can flexibly incorporate a strong \iw baseline as a slice to match its performance. We also show that \name can closely estimate performance drops as large as $23$ accuracy points in an ``in-the-wild'' sentiment analysis setting, which shows the potential of automatic slice function design.\footnote{Code for \name\ can be found at \url{https://github.com/HazyResearch/mandoline}.}

\section{Background}
%
We first provide background on importance weighting and density ratio estimation (Section~\ref{sec:IW}), and the challenges these approaches face.\footnote{A broader discussion of related work is provided in Appendix \ref{app:ext-related}.}
We then provide a formal setup of our problem of evaluating models under distribution shift (Section~\ref{sec:problem}).

\noindent{\bf Notation.} $\p_s$ and $\p_t$ are source and target distributions with respective densities $p_s$ and $p_t$. $\mathbb{E}_s$ and $\mathbb{E}_t$ are expectations with respect to $\p_s$ and $\p_t$. When a statement applies to both distributions, we refer to their densities collectively as $p$.

\subsection{Importance Weighting} \label{sec:IW}

Importance weighting~\citep{Horvitz1952AGO} is a general approach for estimating properties of a target random variable $X$ drawn from $\p_t$ given samples $\{x_i\}_{i = 1}^n$ from a different source distribution $\p_s$. 
Since $\E{s}{\tfrac{p_t(X)}{p_s(X)} f(X)} = \E{t}{f(X)}$ for any function $f$ when $\supp(\p_t) \subseteq \supp(\p_s)$, one can estimate $ \E{t}{f(X)}$ with the empirical average $\tfrac{1}{n} \sum_{i=1}^n\tfrac{p_t(x_i)}{p_s(x_i)} f(x_i)$. Typically $p_s$ and $p_t$ are unknown, so the density ratio $\frac{p_t(X)}{p_s(X)}$ must also be estimated.

\noindent\textbf{Density ratio estimation.} The challenge of how to estimate density ratios is well-studied. 
Estimation of the individual densities to compute the ratios is possible but can lead to poor estimates in high dimensions~\citep{Kpotufe2017LipschitzDS}. 
Instead, most approaches estimate the ratio directly. We discuss several common methods below, although they can all be generalized to fitting the density ratio under the Bregman divergence~\citep{Sugiyama2012DensityratioMU}.
First, classifier-based approaches (CBIW) use a Bayesian ``density-ratio trick''~\citep{Hastie2001TheEO,Sugiyama2012DensityRE,Mohamed2016LearningII} to cast estimation as binary classification between samples from the source $(z=0)$ and target $(z=1)$ distributions. The learned weights $\frac{p(z = 1|x)}{p(z=0 | x)}$ are then rescaled to produce a ratio estimate. 
%
%
Kernel mean matching (KMM) matches moments from $\p_t$ to a (parameterized) transformation of the moments of $\p_s$ in a Reproducing Kernel Hilbert Space (RKHS) e.g. with the Gaussian kernel~\citep{Gretton2009CovariateSB}. 
Least-squares importance fitting (LSIF) directly fits the density ratio by minimizing the squared error of a parametrized ratio $r_\phi(x)$ (typically linear or kernel model) compared to $\frac{p_t(x)}{p_s(x)}$~\citep{JMLR:v10:kanamori09a}. 
Finally, the Kullback-Leibler importance estimation procedure (KLIEP) uses $r_\phi(x)$ to construct an approximate distribution $\hat{\p}_t = r_\phi \p_s$ and minimizes the KL-divergence between $\hat{\p}_t$ and $\p_t$~\citep{sugiyama2008direct}. 
%

%
%



\noindent\textbf{Challenges for IW.}
A common problem when applying \iw is \emph{high-variance weights}, which result in poor performance both theoretically and in practice~\citep{cortes}. 
While simple techniques such as self-normalization and weight clipping can be used to reduce variance~\citep{Grover2019BiasCO}, these heuristics do not address the cause of the variance.
Instead, we highlight and address two challenges that underlie this problem in \iw---learning from high-dimensional, complex data and handling support shift:

\begin{enumerate}
    \item \it{High-dimensional data.} Dealing with high-dimensional data is challenging in density ratio estimation, as it is difficult to find well-specified model classes (for CBIW) or data representations (for KMM, LSIF, KLIEP). 
    To remedy this, dimensionality reduction can be used when the distribution shift is restricted to some low-dimensional structured subspace~\citep{SUGIYAMA201044}, but these approaches generally assume a linear subspace. 
    \item \it{Support shift.} When there exists some $x$ such that $p_s(x) > 0$ but $p_t(x) = 0$, then $\frac{p_t(x)}{p_s(x)} = 0$. This point is essentially discarded, which reduces the effective number of samples available for \iw, and points in the intersection of the support may also have low $p_s(x)$, which results in overweighting a few samples. When there exists some $x$ such that $p_s(x) = 0$ but $p_t(x) > 0$, this $x$ will never be considered in the reweighting---in fact, $\E{s}{\frac{p_t(X)}{p_s(X)}f(X)}$ will \emph{not} equal $\E{t}{f(X)}$ in this case, rendering IW ineffective in correcting distribution shift. We describe these phenomena as support shift.
\end{enumerate}

\subsection{Problem Formulation} \label{sec:problem}
We are given a fixed model $f_\theta : \X \rightarrow \Y$, a labeled ``validation'' \emph{source} dataset $\D_s = \{(x^s_i, y^s_i)\}_{i=1}^{n_s}$, and an unlabeled \emph{target} dataset $\D_t = \{x^t_i\}_{i=1}^{n_t}$. $\X, \Y$ denote the domains of the $x$'s and $y$'s, and $\D_t$ and $\D_s$ are drawn i.i.d from $\p_t(\cdot | Y)$ and $\p_s$, respectively. We assume there is no concept drift between the distributions, meaning $p_t(Y | X) = p_s(Y | X)$. Define $\loss: \X \times \Y \rightarrow \R$ to be a metric of performance of $f_\theta$. Our goal is to evaluate performance on the target population as $\Lt = \E{t}{\ell_\theta(X, Y)}$ using labeled samples from $\D_s$ and unlabeled samples from $\D_t$. (In our experiments, we use $\ell_\theta(X, Y) = \mathbf{1}(f_\theta(X) = Y)$ unless otherwise specified; i.e., our goal is to estimate the performance of classification model $f_\theta$ on the target dataset.)

Standard IW can estimate $\Lt$ using the density ratio of the features as $\frac{1}{n_s} \sum_{i = 1}^{n_s} \frac{p_t(x_i^s)}{p_s(x_i^s)} \ell_\theta(x_i^s, y_i^s)$ using the assumption of no concept drift. However, IW has shortcomings (as discussed in Section \ref{sec:IW}).
For instance, it is often \emph{beneficial} to ignore certain features, as the below example illustrates.

\paragraph{Example 2.3}
\label{ex:iw}
Suppose $p_s(x_1, x_2) \propto \phi_{(\mu_1, \sigma_1^2)}(x_1) \cdot \mathbf{1}(x_2 \in (0, 1))$ and $p_t(x_1, x_2) \propto  \phi_{(\mu_2, \sigma_2^2)}(x_1) \cdot \mathbf{1}{(x_2 \in (-1, 0))}$, where $\phi_{\mu,\sigma^2}$ is the $\N(\mu, \sigma^2)$ density. So, $x_1$ is normal while $x_2$ is uniformly distributed on $(0, 1)$ and $(-1, 0)$ under $p_s$ and $p_t$ respectively, and $x_1, x_2$ are independent. Suppose $\loss(x, y)$ is independent of $x_2$. The IW estimate of $\E{t}{\loss(X, Y)}$ will always be zero since $p_t(x) = 0$ for all $x$ in the support of $p_s$. IW fails due to support shift of the irrelevant feature $x_2$, but this can be mitigated by weighting only on $x_1$. This example helps motivate our framework.


\section{The \name Framework} \label{sec:alg}

We present a framework of assumptions (Section~\ref{subsec:model-dist-shift}) on $\p_s$ and $\p_t$ that motivate user-specified \textit{slicing functions}, which intend to capture relevant distribution shift. 
Under these assumptions, we show that weighting based on accurate slicing functions is equivalent to weighting based on features, but mitigates the challenges that standard \iw faces by ignoring irrelevant and non-shifting components of the distributions (Section~\ref{subsec:iw-sfs}). 
We then present a novel density ratio estimation algorithm based on KLIEP (Section~\ref{subsec:density-ratio-estimation}) that accounts for noisy slices. 

\subsection{Modeling Distribution Shift}
\label{subsec:model-dist-shift}

Assume that each sample $X$ can be represented via mappings to four sets of variables $g(X)$, $h(X)$, $a(X)$, $b(X)$.
This categorizes information about the data depending on if it is relevant to the learning task and if its distribution changes between $\p_s$ and $\p_t$. $g(X)$ contains relevant properties of the data that are known to the user and undergo distribution shift. $h(X)$ represents ``hidden'' properties of the data that are also relevant and shifting, but which the user fails to model. $a(X)$ corresponds to the properties that exhibit shift but are irrelevant to the task. Lastly, $b(X)$ are the properties that do not undergo any shift. We 
state these assumptions formally below.

\begin{assumption}[Shift and relevance of data properties]
\label{assump:mand}
\leavevmode
\begin{enumerate}
    \item Representation of $X$ using $g, h, a, b$: $p_s(X | g, h, a, b) = p_t(X | g, h, a, b) = 1$; in other words, $X$ is (almost surely) fully known given $g,h,a,b$.
    \label{assump:rep}
    \item Shift along $g, h, a$ only: $p_s(X | g, h, a) = p_t(X | g, h, a)$.
    \label{assump:shift}
    \item No support shift on $g$: $p_s(g) = 0 \Rightarrow p_t(g) = 0$.
    \label{assump:supp}
    \item Irrelevance of $a(X)$ to other features and true/predicted labels: $a \indep b\, |\, g, h$, $a \indep Y\, |\, g, h, b$, and $a \indep \loss(X,Y) \, |\, g, h, b$, for both $\p_s$ and $\p_t$.
    \label{assump:rel}
\end{enumerate} 
\label{a:distributions}
\end{assumption}

\noindent $g(X)$ encapsulates the user's beliefs of what axes the shift between $\p_s$ and $\p_t$ occurs on. Since $g(X)$ may be difficult to model precisely, the user approximates them by designing $k$ \textit{slicing functions} $\gtilde(X) = \{\gtilde_1(X), \dots, \gtilde_k(X) \}$, where each $\gtilde_i: \X \rightarrow \{-1, 1\}$ noisily captures an axis $g_i(X)$ via a binary decision. When $h(X)$ is empty, we say that $\gtilde(X)$ is \emph{fully specified}, and otherwise that $\gtilde(X)$ is \emph{underspecified}. When $\gtilde(X) = g(X)$, we say that the slices are \emph{perfect} and otherwise \emph{noisy}. 

We note that the decomposition $g,h,a,b$ need not be unique, and a decomposition of this form always exists: for instance, we can trivially set $h(X) = X$ and $g,a,b$ to be empty. However, to obtain better guarantees, $h$ should ideally contain ``as little information as possible.''


\subsection{Importance Weighting Based on Slicing Functions}
\label{subsec:iw-sfs}
Based on these assumptions, weighting using the relevant shifting properties $g$ and $h$ is sufficient. We state this more formally in Proposition \ref{prop:ghij}:
\begin{proposition}
By Assumption \ref{a:distributions}, $$\Lt = \E{s}{\frac{p_t(g(X), h(X))}{p_s(g(X), h(X))} \loss(X, Y)}.$$ \label{prop:ghij}
\end{proposition} 

\noindent Revisiting Example 2.3 with $g(x) = x_1$, $a(x) = x_2$, and no $h(x)$, $b(x)$, Proposition \ref{prop:ghij} confirms our intuition that weighting on $x_1$ is sufficient and reduces support shift.

If $\gtilde$ is perfect and well-specified, using  $\frac{p_t(\gtilde(x))}{p_s(\gtilde(x))}$ for weighting corrects the distribution shift without being susceptible to extreme shifts in $a(x)$ and the dimensionality and noise added by $b(x)$. In the more frequent case when $h(x)$ is nonempty, reweighting based on $g(x)$ to estimate $\Lt$ already yields a biased approximation $\Lg = \E{s}{\frac{p_t(g(X))}{p_s(g(X))} \loss(X, Y)}$. However, as long as the slices are not noisy, the density ratio methods discussed in Section~\ref{sec:IW} can be applied on the slices with well-studied tradeoffs in computational efficiency, estimation error, and robustness to misspecification~\citep{kanamori2010}. When the slices are noisy, the challenge is to learn weights $w(x) = \frac{p_t(g(x))}{p_s(g(x))}$ on $g$ when we only have $\gtilde$ and our datasets, motivating our algorithm for this particular case.




\subsection{Density Ratio Estimation Approach}
\label{subsec:density-ratio-estimation}
We estimate the density ratio as $\hat{w}(x)$ using $\gtilde$ and our data. We partition $\D_s$ into $\D_{s_1}$ and $\D_{s_2}$ of sizes $n_{s_1}, n_{s_2}$ such that the former is used to learn $\hat{w}(x)$ and the latter is used for evaluation. Our estimate of $\Lt$ is $\Lghat \defeq \frac{1}{n_{s_2}} \sum_{i = 1}^{n_{s_2}} \hat{w}(x_i^{s_2}) \loss(x_i^{s_2}, y_i^{s_2})$. We present a noise-aware extension of LL-KLIEP~\citep{Tsuboi2009} that models $g$ as a latent variable and $p(\gtilde, g)$ as a graphical model. While previous work on latent variable density ratio estimation focuses on the posterior distribution and requires observations of $g$~\citep{liu2020posterior}, our algorithm allows for denoising by incorporating a user's prior knowledge on the quality of $\gtilde$, which can be viewed as hyperparameters of user confidence. 

\noindent\textbf{Graphical Model.} Let a graph $G = (\gtilde, E)$ specify dependencies between the slicing functions using standard notions from the probabilistic graphical models literature~\citep{Lauritzen,koller2009probabilistic}. 
We assume the user knows dependencies between $\gtilde$, although $E$ can be learned~\citep{ravikumar2010, Loh_2013}, and assume each $\gtilde_i$ is connected to at most one other $\gtilde_j$.
For the joint distribution of $(g, \gtilde)$, we augment $G$ by adding an edge from each $\gtilde_i$ to $g_i$ in the following model for both $p_s$ and $p_t$:

{\small
\begin{align}
    p(\gtilde, g; \theta) = \frac{1}{Z_\theta} \exp \bigg[\sum_{i = 1}^k \theta_i g_i + \sum_{i = 1}^k \theta_{ii} g_i \gtilde_i + \sum_{\mathclap{(i, j) \in E}} \theta_{ij} \gtilde_i \gtilde_j \bigg], \label{eq:model_joint}
\end{align}
}

\noindent where $Z_\theta$ is the log partition function. Note that when $\gtilde = g$, this reduces to an Ising model of $g$ with edgeset $E$. In Appendix~\ref{app:alg} we show that the marginal density of $g$ is then
{\small
\begin{align}
    p(g; \psi) = \frac{1}{Z} \exp \bigg[\sum_{i = 1}^k \psi_i g_i + \sum_{(i, j) \in E} \psi_{ij} g_i g_j \bigg]  \label{eq:model_g}
    = \frac{1}{Z} \exp(\psi^\top \phi(g)),
\end{align}
}
\newline
\noindent where $Z$ is a different log partition function, $\phi(g)$ is a representation of the potentials over $g$, and each element of $\psi$ is a function of $\theta$ in \eqref{eq:model_joint}. Define $\delta = \psi_t - \psi_s$ as the difference in parameters of $p_t(g)$ and $p_s(g)$. Under this model, the density ratio to estimate is $w(x) = \frac{p_t(g(x); \psi_t)}{p_s(g(x); \psi_s)} = \exp \big(\delta^\top \phi(g(x))\big) \frac{Z_s}{Z_t}$.

\renewcommand{\E}[2]{\mathbb{E}_{#1}\left[#2\right]}

\noindent \textbf{Latent Variable KLIEP.} Based on the parametric form of $w(x)$, KLIEP aims to minimize the KL-divergence between the target distribution $\p_t$ and an estimated distribution $\hat{\p}_t$ with density $\hat{p}_t(g; \delta) = \exp \big(\delta^\top \phi(g)\big) \frac{Z_s}{Z_t} p_s(g)$. Note that since $\hat{p}_t(g; \delta)$ must be a valid density, the log partition ratio $\frac{Z_s}{Z_t}$ is equal to $\frac{1}{\E{s}{\exp(\delta^\top \phi(g))}}$. Then, minimizing the KL-divergence between $\p_t$ and $\hat{\p}_t$ is equivalent to solving
\begin{align}
    \max\limits_{\delta} \left\{\E{t}{\delta^\top \phi(g)} -  \log \E{s}{\exp(\delta^\top \phi(g))} \right\}.
    \label{eq:obj}
\end{align}




\noindent The true distribution of $g$ is unknown, but we can write \eqref{eq:obj} as $\E{t}{\E{t}{\delta^\top \phi(g) | \gtilde}} - \\
\log \E{s}{\E{s}{\exp(\delta^\top \phi(g)) | \gtilde}}$.
The outer expectation over $\gtilde$ can be approximated empirically, so in place of \eqref{eq:obj} we want to maximize $\frac{1}{n_t} \sum_{i = 1}^{n_t} \E{t}{\delta^\top \phi(g) | \gtilde(x_i^t)} - \log \big(\sum_{j = 1}^{n_{s_1}} \E{s}{\exp(\delta^\top \phi(g)) | \gtilde(x_j^{s_1})} \big)$. 

\noindent \textbf{Noise Correction.} This empirical objective function requires knowledge of $p(g | \gtilde)$, which factorizes as $\prod_{i = 1}^k p(g_i | \gtilde_i)$. This inspires our noise-aware KLIEP approach: users provide simple $2 \times 2$ \textit{correction matrices} $\sigma_s^i, \sigma_t^i$ per slice to incorporate their knowledge of slice quality, where $\sigma_s^i(\alpha, \beta) \approx {p_s(g_i = \alpha | \gtilde_i = \beta)}$
and similarly for $\sigma_t^i$. This knowledge can be derived from prior ``evaluation'' of $\gtilde_i$'s and can also be viewed as a measure of user confidence in their slices. Note that setting each $\sigma^i$ equal to the identity matrix recovers the noiseless case $\gtilde = g$, which is LL-KLIEP. Our convex optimization problem maximizes
{\small
\begin{align}
    \hatfkliep(\delta, \sigma) &= \frac{1}{n_t} \sum_{i = 1}^{n_t} \Enoisy{t}{\delta^\top \phi(g) | \gtilde(x_i^t)} \label{eq:obj_hat}
    - \log \bigg(\sum_{j = 1}^{n_{s_1}} \Enoisy{s}{\exp(\delta^\top \phi(g)) | \gtilde(x_j^{s_1})} \bigg),
\end{align}
}

\noindent where $\Enoisy{s}{r(g) | \gtilde} = \int r(g) \prod_{i = 1}^k \sigma_s^i(g_i, \gtilde_i) dg$ for any function $r(g)$. Define $\dhat = \argmax{\delta}{\hatfkliep(\delta, \sigma)}$. Then the estimated density ratio is $\hat{w}(g)  = \frac{n_{s_1} \exp \left(\dhat^\top \phi(g) \right)}{\sum_{i = 1}^{n_{s_1}} \Enoisy{s}{\exp(\dhat^\top \phi(g)) | \gtilde(x_i^{s_1})}}$, for which we have used an empirical estimate of $\frac{Z_s}{Z_t}$ as well. To produce weights on $\D_s$, we define $\hat{w}(x) = \Enoisy{s}{\hat{w}(g) | \gtilde(x)}$. Our approach is summarized in Algorithm \ref{alg:noise_kliep}. 

\begin{algorithm}[t]
	\caption{\name}
	\label{alg:noise_kliep}
\SetAlgoLined
		\textbf{Input}: Datasets $\D_s, \D_t$; slicing functions $\gtilde: \X \rightarrow \{-1, 1\}^k$, dependency graph $G = (\gtilde, E)$, \newline
		correction matrices $\sigma_s^i$ and $\sigma_t^i$ for each slice \\
		\nl Split $\D_s$ into $\D_{s_1}, \D_{s_2}$ \\
		\nl Use $G$'s edgeset to construct representation function $\phi$ \\
		\nl Solve $\dhat = \argmax_\delta {\hatfkliep(\delta, \sigma)}$ defined in \eqref{eq:obj_hat} \\
		\nl Construct ratio $\hat{w}(g) = \frac{n_{s_1} \exp(\dhat^\top \phi(g))}{\sum_{j = 1}^{n_{s_1}} \Enoisy{s}{\exp(\dhat^\top \phi(g) | \gtilde(x_j^{s_1}))}}$ \\
		\nl \Return{$\Lghat = \frac{1}{n_{s_2}}\sum_{i = 1}^{n_{s_2}} \Enoisy{s}{\hat{w}(g) | \gtilde(x_i^{s_2})} \loss(x_i^{s_2}, y_i^{s_2})$}
\end{algorithm}

Note that our approach can handle \textit{incomplete} slices that do not have full coverage on the dataset. We model $\gtilde_i$ with support $\{-1, 0, 1\}$ where $0$ represents abstention; this can be incorporated into \eqref{eq:model_joint} as described in Appendix~\ref{app:alg}.

\section{Theoretical Analysis} \label{sec:theory}

We analyze the performance of our approach by comparing our estimate $\Lghat$ to the true $\Lt$. We show the error can be decomposed into a bias dependent on the user input and a variance dependent on the true distribution of $g$, noise correction, and amount of data. We provide an error bound that always holds with high probability, in contrast to standard IW for which no generalization bounds hold for certain distributions. 


\noindent Define a ``fake'' $g$ estimated from inaccurate $\sigma$ as $\gbar$, where $p(\gbar) = \int \sigma(g, \gtilde) p(\gtilde) d\gtilde$, and $\Lgbar = \E{s}{w(\gbar(X)) \loss(X, Y)}$, where $w(\gbar(X)) {=} \frac{p_t(\gbar(X))}{p_s(\gbar(X))}$. Then,
\begin{align}
    |\Lt - \Lghat|  \le  \underbrace{|\Lt - \Lg|}_{\mathclap{\text{bias from no $h(X)$}}} \qquad+ \qquad \underbrace{|\Lg - \Lgbar |}_{\mathclap{\text{bias from incorrect $\sigma$}}} \qquad
    +  \qquad \underbrace{|\Lgbar - \mathbb{E}_{s}[\Lghat]|}_{\mathclap{\text{var. from estimated ratio}}} \qquad + \qquad \underbrace{|\mathbb{E}_{s}[\Lghat] - \Lghat|}_{\mathclap{\text{var. from empirical evaluation}}}. \label{eq:decomposition} 
\end{align}


Suppose all slices are noisy and the user fails to correct $k' \le k$ of them. Define $\eta_s^{\min}(i), \eta_s^{\max}(i)$ as bounds on the relative error of $\sigma_s^i$ such that $|\frac{p_s(g_i | \gtilde_i) - \sigma_s^i(g_i, \gtilde_i)}{p_s(g_i | \gtilde_i)}| \in [\eta_s^{\min}(i), \eta_s^{\max}(i)]$ for all $g_i, \gtilde_i$ per slice. $\eta_t^{\min}(i)$ and $\eta_t^{\max}(i)$ are similarly defined per slice for $\p_t$, and define the total correction ratio as $r = \prod_{i = 1}^{k'} \frac{1 + \eta_t^{\max}(i)}{1 - \eta_{s}^{\min}(i)}$. Define an upper bound $M = \sup_X \frac{p_t(g(X))}{p_s(g(X))}$ on the density ratio of $g$, and define $\Mhat = \sup_X \hat{w}(g(X))$. Then we have the following Theorem \ref{thm:main} (proved in Appendix \ref{mand:thmmain_prf}).

\begin{theorem}
Set $n_{s_1} = n_{s_2} = \frac{n_s}{2}$. Under Assumption \ref{a:distributions}, with probability at least $1 - \varepsilon$, the accuracy of our estimate of $\Lt$ via noise-aware reweighting is bounded by
{\footnotesize
\begin{align*}
    &|\Lt - \Lghat| \le \hspace{-0.3em} \I{}{} \hspace{-0.25em} \Big|p_t(h(x)|g(x)) {-} p_s(h(x)|g(x))\Big| \cdot p_t(g(x))\,\dx  
    + rM \sum_{i = 1}^{k'} \lt \tfrac{\eta_t^{\max}(i)}{1 - \eta_t^{\min}(i)} + \tfrac{\eta_s^{\max}(i)}{1 - \eta_s^{\min}(i)}\rt
    + \hat{M} (c_{s, \loss} {+} 1 ) \sqrt{\tfrac{\log \lt \tfrac{4}{\varepsilon}\rt}{n_s}},
\end{align*}
}

where $c_{s, \loss}$ is a constant dependent on the distribution of $\loss$ and $\p_s$, and  $\hat{M} \xrightarrow{p} r M$ as $n_s \wedge n_t \rightarrow \infty$.
\label{thm:main}
\end{theorem}
\noindent We make two observations about Theorem \ref{thm:main}.

\begin{itemize}
    \item The first two terms are the bias from user input and map to the first two terms in \eqref{eq:decomposition}. The first ``total variation-like'' term is a bias from not modeling $h(X)$, and describes relevant uncaptured distribution shift. The second term describes the impact of inaccurate correction matrices $\sigma_s, \sigma_t$.
    \item The last term maps to the third and fourth terms of \eqref{eq:decomposition} and describes the variance from learning and evaluating on the data, which collectively with $\hat{M}$ scales in $n_t$ and $n_s$. This bound depends on the distributions themselves---critically the upper bound $M$ on the weights---and the accuracies of $\sigma_s, \sigma_t$. 
\end{itemize}

\noindent When the user writes perfect, fully-specified slicing functions, the standing bias of the first two terms in Theorem \ref{thm:main} is $0$, and our estimate $\Lghat$ converges to $\Lt$.

\noindent \textbf{Comparison to standard IW.}
We compare our approach to standard \iw in the feature space using our decomposition framework in \eqref{eq:decomposition}. 
Standard \iw does not utilize slices or user knowledge, so the first two terms of \eqref{eq:decomposition} would be $0$. The latter two terms depend on the magnitude of the true weights, and more generally, their variance. In standard \iw, the variance of the weights can be large and even unbounded due to support shift or high-dimensional complex data. Even simple continuous distributions can yield bad weights; for instance, \cite{cortes} shows a Gaussian distribution shift where the variance of the weights is infinite, leading to inapplicable generalization bounds and poor empirical performance. In contrast, while our approach using $\gtilde$ may incur bias, it will always have weights bounded by at most $\hat{M}$ due to the discrete model and hence have applicable bounds. The variance may also be lower due to how we mitigate problems of support shift and high-dimensional complex data by ignoring $a(X), b(X)$. The estimation error bounds for standard IW scale proportionally to $\sqrt{\sigma^2/n_s}$, where $\sigma^2$ denotes the variance of the importance weights \cite{cortes}; therefore, when this quantity outweighs the bias incurred by \name{} due to underspecified, noisy $\gtilde$, \name{} can reduce the estimation error of $\Lt$ compared to standard IW.

\begin{figure}[t!]
    \centering
    \begin{subfigure}{0.4\linewidth}
    \includegraphics[width=\linewidth]{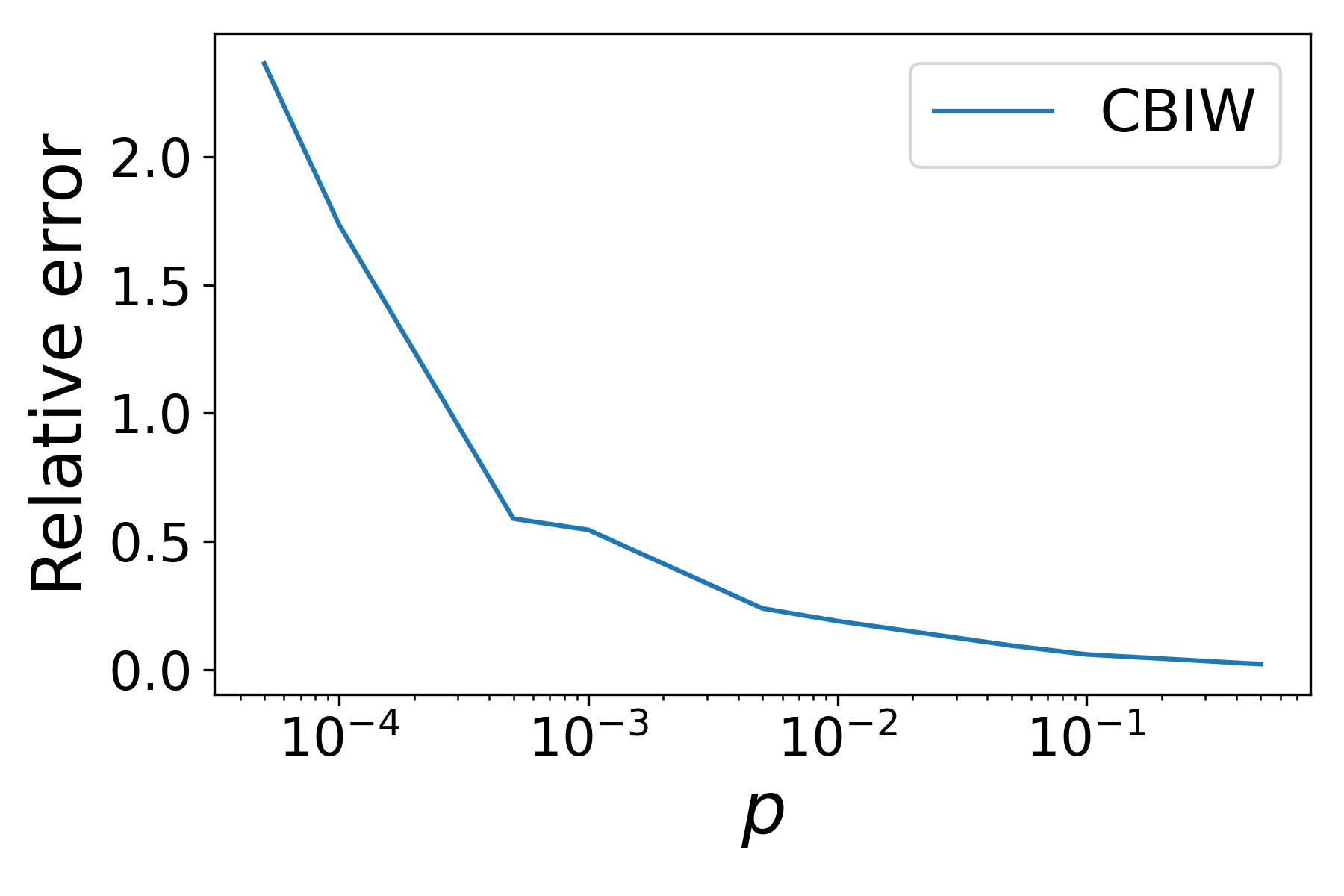}
    \end{subfigure}\qquad %
    \begin{subfigure}{0.4\linewidth}
    \includegraphics[width=\linewidth]{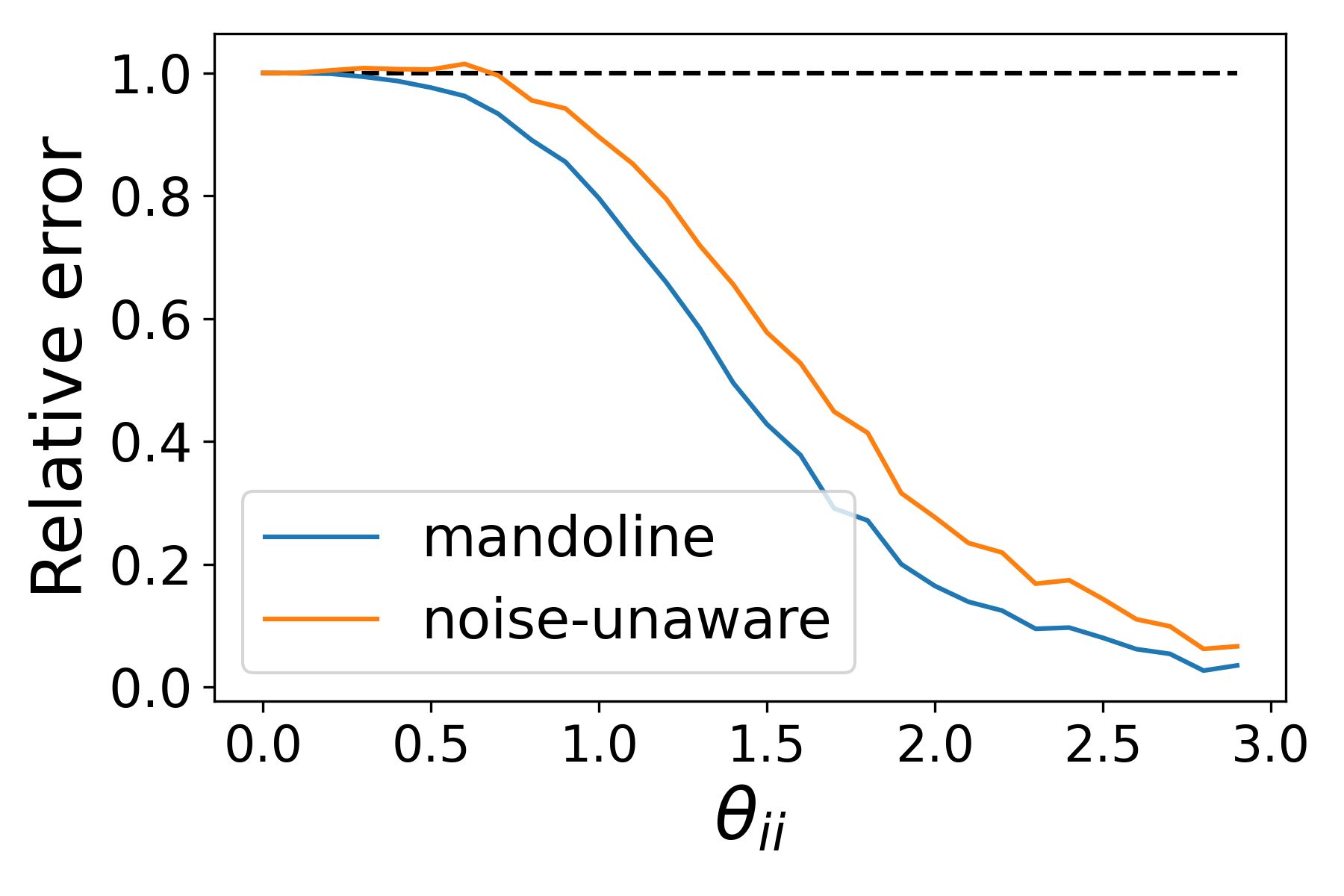}
    \end{subfigure}
    \caption[Synthetic Experiments for CBIW and \name]{\textit{(left)} Relative error of CBIW in the presence of a spurious shifting feature with probability $p$. (\name\ has a relative error of $0.01$.) \textit{(right)} Relative error of \name\ vs. noise-unaware reweighting as correlation between $g$ and $\gtilde$ increases.}
    \label{fig:synthetic}
    
\end{figure}

\subsection{Synthetic Experiments}
\label{sec:synthetics}

In this section, we provide an example where standard importance weighting struggles while \name\ maintains accurate estimates. We then evaluate how \name\ performs as the noise in the $\gtilde_i$'s varies. Additional details and synthetics are included in Appendix~\ref{app:synthetics}.


\noindent{\bf Support Shift.}
Standard \iw struggles when the source and target distributions have little overlap. As described in Example 2.3, the density ratio is 0 or undefined when the supports are disjoint. Moreover, if the source has a low-density region where the target density is high, this effectively reduces the sample size by assigning high weight to a few examples. To illustrate, in Figure 2a we generate data with a single $g_i$ and set $\tilde{g}_i = g_i$. We then append a ``spurious'' feature ($a(X)$) that is 1 with small probability $p$ for source examples, and 1 for \emph{all} target examples. We plot the relative error of CBIW on all features as $p$ decreases to 0; for small values of $p$ this is greater than the difference between source and target. By contrast, \name\ ignores the spurious feature entirely as it is not part of $g$, achieving a relative error of only 0.01.


\noindent{\bf Noisy Slicing Functions.}
In Figure~\ref{fig:synthetic}b, we show that on data generated from \eqref{eq:model_joint}, \name's noise-aware approach does better than the noise-unaware approach (i.e., setting each $\sigma_t^i$, $\sigma_s^i$ to the identity matrix) as we vary the correlation between $g_i$ and $\gtilde_i$ by adjusting $\theta_{ii}$ in \eqref{eq:model_joint}.
\newcommand{\snli}{{\sc SNLI}\xspace}
\newcommand{\hans}{{\sc HANS}\xspace}
\newcommand{\mnli}{{\sc MNLI}\xspace}
\newcommand{\snlip}{{\sc SNLI}\textsuperscript{+}\xspace}
\newcommand{\hansm}{{\sc HANS}\textsuperscript{--}\xspace}
\newcommand{\snliphansm}{{\sc SNLI}\textsuperscript{+}$\rightarrow${\sc HANS}\textsuperscript{--}\xspace}
\newcommand{\celeba}{{\sc CelebA}\xspace}
\newcommand{\civil}{{\sc CivilComments}\xspace}
\newcommand{\snlimnli}{{\sc SNLI}$\rightarrow${\sc MNLI}\xspace}
\newcommand{\snlihans}{{\sc SNLI}$\rightarrow${\sc HANS}\xspace}
\newcommand{\imdb}{{\sc IMDB}\xspace}

\begin{table*}[t!!]
    \centering
    \footnotesize
    \resizebox{\linewidth}{!}{%
    \begin{tabular}{p{0.2\linewidth}|p{0.165\linewidth}|p{0.2\linewidth}|p{0.22\linewidth}|p{0.155\linewidth}|p{0.19\linewidth}}
        \toprule
        {\bf Task} & {\bf Task Labels} & {\bf Distribution Shift} & {\bf Source Data} & {\bf Target Data} & {\bf Slices}  \\
        \midrule
        \makecell[l]{\celeba \\ {\it image classification}} 
        & \makecell[l]{{\it male} \\ {\it vs. female}}
        & \makecell[l]{$\uparrow$ blurry images}
        & \makecell[l]{validation set\\ {\it (1\% blurry)} }
        & \makecell[l]{perturbed test set\\ {\it (30\% blurry)} }
        & \makecell[l]{{\sc metadata labels}\\ {\it blurry / not blurry}}\\
        \cmidrule{2-6}
        \makecell[l]{\civil \\ {\it toxic text classification}} 
        & \makecell[l]{{\it toxic} \\ {\it vs. non-toxic}}
        & \makecell[l]{$\uparrow\downarrow$ identity proportions\\ {\it (e.g. $\uparrow$ female)} }
        & \makecell[l]{validation set }
        & \makecell[l]{perturbed test set }
        & \makecell[l]{{\sc metadata labels}\\ {\it 8 identity groups}} \\
        \cmidrule{2-6}
        \makecell[l]{\snlimnli \\ {\it natural language}\\ {\it inference}} 
        & \makecell[l]{{\it entailment, neutral} \\{\it or contradiction}}
        & \makecell[l]{single-genre $\rightarrow$\\ multi-genre examples}
        & \makecell[l]{\snli validation set }
        & \makecell[l]{\mnli matched\\ validation set }
        & \makecell[l]{{\sc programmatic}\\ {\it task model predictions,}\\ {\it task model entropy}} \\
        \cmidrule{2-6}
        \makecell[l]{\snliphansm \\ {\it natural language}\\ {\it inference}} 
        & \makecell[l]{{\it entailment} \\{\it or non-entailment}}
        & \makecell[l]{$\uparrow$ lexical overlap\\ with label shift }
        & \makecell[l]{\snli validation set + \\ 1\% \hans validation set} 
        & \makecell[l]{99\% \hans \\validation set\\} 
        & \makecell[l]{{\sc programmatic}\\ {\it lexical overlap}\\ {\it noisy non-entailment}\\ {\it sentence structure}\\ {\it task model entropy}}\\
        \bottomrule
    \end{tabular}%
    }
    \caption[\name{} Evaluation Tasks]{Summary of real-world tasks and datasets considered.}
    \label{tab:tasks_and_datasets}
\end{table*}

\section{Experiments on Real Data}
\label{sec:experiments}
We empirically validate that \name can estimate performance under distribution shift on real tasks and datasets.

\noindent {\bf Experimental Claims.} We validate claims about \name~under the experimental settings described below:

\begin{enumerate}
    \item {\bf Fully-specified, with perfect slices (Section~\ref{sec:fullspec-perfectslices}).} In the setting when all the factors of distribution shift are known (i.e., $h(X)$ is empty), and the true slice labels are known via metadata annotations (i.e., $\gtilde = g$), \name reduces the model performance estimation error by up to $3\times$ over baselines.
    \item {\bf Fully-specified, with noisy slices (Section~\ref{sec:fullspec-noisyslices}).} When the slices are noisy / programmatic, but still capture all the shifted variables (i.e., $h(X)$ is still empty but $\gtilde \ne g$), \name's performance does not significantly degrade, and it remains competitive with baselines.
    \item {\bf Underspecified, with noisy slices (Section~\ref{sec:underspec-noisyslices}).} For large distribution shifts where only a subset of relevant shifted variables are captured by the noisy programmatic slices (i.e. $h(X)$ is nonempty and $\gtilde \ne g$), \name reduces estimation error by up to $2\times$ over baselines. 
\end{enumerate}





\noindent {\bf Tasks.} We consider four tasks from computer vision and natural language processing, summarized in Table~\ref{tab:tasks_and_datasets}.

    
    

\newcommand{\kmm}{{\sc KMM}\xspace}
\newcommand{\ulsif}{{\sc uLSIF}\xspace}
\newcommand{\cbiwft}{{\sc CBIW-FT}\xspace}
\newcommand{\llkliep}{{\sc LL-KLIEP}\xspace}
\newcommand{\sest}{{\sc Source}\xspace}
\noindent{\bf Baselines.}
We compare \name against widely used importance weighting baselines on the features.
     {\bf Direct Source Estimation (\sest)} directly uses the estimate from the source distribution for the target.
     {\bf Classifier-Based Importance Weighting (\cbiw)} uses logistic regression on a task-specific feature representation to distinguish samples from the source and target distributions. As noted by \cite{Tsuboi2009}, this baseline returns weights identical to LL-KLIEP on features.
{\bf Kernel Mean Matching (\kmm)} \citep{Gretton2009CovariateSB} solves a quadratic program to minimize the discrepancy in feature expectations between the target and a reweighting of the source distribution. {\bf Unconstrained Least-Squares Importance Fitting (\ulsif)}~\citep{JMLR:v10:kanamori09a} is a stable version of LSIF that solves a least-squares optimization problem for the density ratio using a linear kernel model. We also run the above baselines on the \emph{slices} to compare against \name's instantiation with KLIEP in the Appendix. 

We self-normalize importance weights for all methods to sum to 1 \citep{Owen2013MC}.
Full experimental details for all experiments are provided in Appendix~\ref{app:experiments}.








\subsection{Fully Specified with Perfect Slices}
\label{sec:fullspec-perfectslices} 
First, we consider distribution shift along variables for which annotated metadata is available, i.e. the identities of all the relevant slices are known, and the ground-truth slice labels are known for each example. 
We consider two tasks in this section: \celeba~(images) and \civil~(text). On \celeba, \name reduces estimation error by $3\times$ over the best baseline (\cbiw). On \civil, \name~and \cbiw~both exhibit strong performance, returning estimates within $0.12\%$ and $0.03\%$ respectively of the true test accuracy. We describe these experiments below, with results summarized in Table~\ref{tab:celeba-civil}.

\subsubsection{CelebA}
\textbf{Task.} In \celeba, we classify faces as either male or female. Each image comes with metadata annotations. We induce distribution shift by perturbing the \celeba~validation and test sets so that the test set has more images annotated as ``blurry'' ($30\%$ of test set vs. only $1\%$ of validation set). 

\noindent \textbf{Models.} We evaluate ResNet-18 and ResNet-50 models pretrained on ImageNet and finetuned for 5 epochs on CelebA. 

\noindent \textbf{Results.}
For the ResNet-50 models, accuracy is significantly higher on the non-blurry images ($97.5\%$) than for the blurry ones ($91\%$); the target accuracy is thus $95.6\%$, vs. $97.4\%$ on the source set. 
We use \name~to estimate the test performance, using the provided ``blurry'' metadata label as our only slice function $g_i$.
The resulting estimate is within $0.16\%$ on average of the true value. 
In comparison, the mean absolute error of the CBIW estimate is $0.53\%$, and $1.76\%$ for KMM and \ulsif.
Table~\ref{tab:celeba-civil} summarizes these results, along with those for ResNet-18, which exhibit similar trends.

\subsubsection{CivilComments}
\label{sec:civil-metadata}
\textbf{Task.} The \civil dataset \citep{borkan2019nuanced} contains comments labeled ``toxic'' or ``non-toxic'', along with 8 metadata labels on whether a particular identity (male, female, etc.) is referenced in the text.
We modify the test set to introduce shift by randomly subsampling examples for each ``slice'' (subset of data with a given assignment of metadata labels), with different proportions per slice. 


\noindent\textbf{Models.} We use a standard {\tt bert-base-uncased} model, fine-tuned on \civil for 5 epochs. 

\noindent\textbf{Results.} When the true $g_i$'s are used, \name~returns accuracy estimates that are within 0.12\% of the true test accuracy. \cbiw~returns an even better estimate (within 0.03\%), while the estimation error of \kmm is 1.25\% and of \ulsif is 0.39\%. By contrast, the raw (unweighted) validation accuracy differs from the test accuracy by 1.6\%; thus, both \name~and the baselines improve this estimate.

%

\begin{table}[b]
    \centering
    \begin{sc}
    \begin{tabular}{c|c|c|c}
        \toprule
         \multirow{3}{*}{Method} & \multicolumn{3}{c}{Average Estimation Error (\%)}\\
         \cmidrule{2-4}
          & \multicolumn{2}{c}{\celeba} & \multicolumn{1}{c}{\civil}\\
          & ResNet-18 & ResNet-50 & BERT\\
         \midrule
        {\sc Source}    & $1.96\%$     & $1.74\%$     & $1.62\% $ \\
        \midrule
        \cbiw           & $0.47\%$     & $0.53\%$     & $\textbf{0.03}\% $ \\
        \kmm            & $1.97\%$     & $1.76\%$     & $1.25\% $ \\
        uLSIF           & $1.97\%$     & $1.76\%$     & $0.39\% $\\
        \name           & $\textbf{0.16}\%$     & $\textbf{0.16}\%$    & $0.12\% $ \\
        \bottomrule
    \end{tabular}%
    \caption[\name{} Mean Absolute Estimation Error for Target Accuracy on \celeba and \civil]{Mean absolute estimation error for target accuracy on \celeba and \civil.}
    \label{tab:celeba-civil}
    \end{sc}
\end{table}




\subsection{Fully Specified with Noisy Slices}
\label{sec:fullspec-noisyslices}
Next, we examine the effect of using noisy, user-provided slicing functions in place of perfect metadata labels. Here, we show that this additional noise does not increase estimation error for \name.

We consider the \civil~task described in the previous section. Instead of using the annotated metadata as our $g_i$'s, we write heuristic slicing functions as a substitute, as one would do in practice if this metadata was unavailable.
For each of the identities described in Section~\ref{sec:civil-metadata}, we write a noisy slice $\gtilde_i$ that detects the presence of this identity. For example, we detect the male identity if the words ``male'', ``man'', or ``men'' appear in the text. These simple slicing functions are reasonably accurate compared to the true metadata annotations (average $0.9$ F1 score compared to metadata across all 8 slices).
With these noisy slices, the accuracy estimate returned by \name~is within \textbf{0.10\%} of the true value, \emph{lower} than the error when using metadata (0.16\%).
This suggests that the noise in our slices is not a major issue, and that in this case our slices in fact better capture the shifts relevant to evaluating the model on the target data.
%

\subsection{Underspecified with Noisy Slices}
\label{sec:underspec-noisyslices}
Next, we turn to an underspecified setting, where the distribution shift is imperfectly captured by programmatic slices. 
Here, \name is able to reduce average estimation error by up to $2\times$ over the best baseline (Table~\ref{tab:snlimnli}).

Concretely, for natural language inference, we study whether it is possible to estimate performance for the \mnli~\citep{Williams2018ABC} validation set (target) using the \snli~\citep{bowman-etal-2015-large} validation set (source). 
The distribution shift from \snlimnli is substantial, as \mnli was designed to capture far more input variability.
%



\begin{table}[t]
    \centering
    \begin{sc}
    \begin{tabular}{c|c|c|c|c}
        \toprule
         \multirow{2}{*}{Method} & \multicolumn{2}{c}{Standard Accuracy} & \multicolumn{2}{c}{Binary Accuracy}\\
          & Avg. Error & Max. Error & Avg. Error & Max. Error\\
         \midrule
        {\sc Source} & $6.2\% \pm 3.8\%$ & $15.6\%$ & $3.0\% \pm 2.3\%$ & $9.3\%$  \\
        \midrule
        \cbiw & $5.5\% \pm 4.5\%$  & $17.9\%$ & $3.7\% \pm 2.5\%$ & $9.8\%$  \\
        \kmm & $5.7\% \pm 3.6\%$ & $14.6\%$ & $3.3\% \pm 2.3\%$ & $8.7\%$  \\
        \ulsif & $6.4\% \pm 3.9\%$ & $16.0\%$ & $3.7\% \pm 2.4\%$ & $9.1\%$ \\
        \name & ${\bf 3.6\% \pm 1.6\%}$ & ${\bf 5.9\%}$ & ${\bf 1.6\% \pm 0.7\%}$ & ${\bf 2.7\%}$  \\
        \bottomrule
    \end{tabular}%
    \end{sc}

    \caption[\name{} Accuracy Estimates for \snlimnli]{Estimating standard and binary accuracy on {\sc MNLI} using {\sc SNLI}. Average and maximum estimation errors for $8$ models are reported, with $95\%$ confidence intervals.}
    \label{tab:snlimnli}
\end{table}


%

\noindent{\bf Programmatic Slices.} 
Since the shift from \snlimnli is large, we design slices that can capture general distributional shifts. 
We construct $3$ slicing functions that check (respectively) whether the evaluated classifier predicted one of the $3$ classes, and $6$ slicing functions that bucket examples based on the entropy of the classifier's outputs. 
These slices capture how the model's uncertainty or predictions change, which are useful indicators for detecting distributional shift~\citep{Hendrycks2017ABF}. Here, they allow us to perform a \emph{model-specific} estimate adjustment.

\begin{table}[b!]
    \centering
    \begin{tabular}{l|c}
        \toprule
        {\bf Slices} & {\bf Average Estimation Error}\\
        \midrule
        {\it Lexical Overlap} & $16.9\%$ \\
        + {\it Noisy Non-Entailment} & $11.8\%$  \\
        + {\it Sentence Structure}  & $6.9\%$  \\
        + {\it Model Entropy} &  $5.9\%$ \\
        + {\it \cbiw Slice} & $0.6\%$ \\
        \bottomrule
    \end{tabular}%
    \caption[\name{} \snliphansm Slicing Ablation]{Reduction in \name estimation error for \snliphansm, as more slices are added to capture distribution shift.}
    \label{tab:hans}
\end{table}

\noindent{\bf Results.} 
We compare $8$ off-the-shelf models taken from the HuggingFace Model Hub. Table~\ref{tab:snlimnli} contains detailed results for estimating both standard accuracy over the $3$ NLI classes, as well as an oft-used binary accuracy metric that combines the contradiction and neutral classes.
For both metrics, \name provides substantially better estimates than baselines while requiring no expensive additional training or fine-tuning. For binary accuracy, \name is able to estimate performance on the MNLI dataset with an average error of only $1.6\%$, when all baselines are \emph{worse} than just using the unadjusted source estimates. 
\name's estimates  are also robust across the evaluated models, with a maximum error that is $3.2\times$ lower than the best baseline.

\noindent{\bf Application: Model Selection.} We check that \name can be used to rank models in terms of their target performance. 
On $7/8$ models, \name correctly assesses if the model improves or degrades (see Appendix~\ref{app:experiments}; Table~\ref{tab:snlimnli-perf-by-model}). 
\name has a Kendall-Tau score of $0.786$ ($p$-value 0.006) to the true \mnli performance---only confusing rankings for the top-$3$ models, which are closely clustered in performance on \mnli ($1\%$ performance spread).

\section{Slice Design}
A natural question raised by our work is: how should we best design slicing functions? 
Key desiderata for slices are that they should be task-relevant and capture important axes of distribution shift. 
For many tasks, candidate ``slicing functions'' in the form of side information are readily available. 
Here, slices may come directly in the form of metadata (such as for \celeba), or as user-defined heuristics (such as the keyword matching we use for \civil) and are already used for evaluation and monitoring purposes.

There are also software tools to create slices~\citep{goel2021robustness}, and a growing body of work around engineering, discovering and utilizing them \citep{chen2019slice, mccoy-etal-2019-right, ribeiro-etal-2020-beyond, wang-etal-2018-glue, hashimoto2021model}, including automated slice discovery methods~\citep{polyzotis2019slice,sliceline21,tae2021slice}. 
Our central contribution is to describe how to model and utilize this noisy side information in order to address long-standing challenges in importance weighting.

We investigate how slice design affects the performance of \name. Section~\ref{sec:underspec-degradation} discusses a challenging underspecified setting, where along with domain knowledge, \name can incorporate standard \iw as a slice to achieve a ``best-of-both-worlds'' performance. Section~\ref{sec:slice-design-in-the-wild} shows that a simple, automated slice design strategy works well out-of-the-box for sentiment classification. These case studies emphasize that understanding slice design is an important direction for future work.
%
%

\subsection{Tackling Underspecification with Slice Design} 
\label{sec:underspec-degradation}
We use a challenging setting to highlight how careful slice design can tackle underspecification. 
Concretely, we consider distribution shifts when moving from \snli~\citep{bowman-etal-2015-large} to \hans~\citep{McCoy2019RightFT}. 
\hans was created to address the lack of diverse \emph{lexical overlap} examples in \snli, i.e. examples where the hypothesis is created using a subset of words from the premise.
Among \snli examples with lexical overlap, only $1.7\%$ are labeled non-entailment. By contrast, \hans only contains lexical overlap examples, with a 50/50 class balance.

Due to the lack of non-entailment lexical overlap examples, estimating performance on \hans using \snli is extremely challenging.
Therefore, we construct a mixture of \snli and \hans, moving $1\%$ of \hans to \snli to create a new source dataset (\snlip), and keeping the remaining $99\%$ of \hans as the target data (\hansm).

In Table~\ref{tab:hans} we examine how changing the set of slices affects \name's performance. Just using a single lexical overlap slice yields high estimation error ($16.9\%$), since it does not adjust for the shift in the proportion of non-entailment lexical overlap ($1.7\% \rightarrow 50\%$). 
To capture this shift, we add noisy slices for contradiction (based on negation and token ordering), sentence structure (word substitutions, length differences, verb tense) and the evaluated model's uncertainty (similar to \snlimnli). These additions further reduce estimation error by $2.86\times$.

Interestingly, \cbiw---when fine-tuned with a {\tt bert-base-uncased} model---learns a classifier that perfectly separates the HANS examples added to \snlip, giving low estimation error ($1.2\%$).
The flexibility of \name allows us to take advantage of this by directly incorporating the \cbiw predictions as a slicing function, giving us a ``best-of-both-worlds'' that achieves extremely low estimation error ($0.6\%$). This highlights a natural strength of \name in being able to easily incorporate information from other methods.

\subsection{Slice Design ``In-the-Wild''} 
\label{sec:slice-design-in-the-wild}
Using sentiment classification on \imdb~\citep{maas2011learning}, we show that automated slice design can be effective out-of-the-box, \emph{without tuning \name at all}. 
We use only the task-agnostic entropy-based slices described in Section~\ref{sec:underspec-noisyslices}.
%
%
Across $3$ models, Table~\ref{tab:imdb-sentiment} shows that we get good estimates when moving from \imdb to varied sentiment datasets. This includes a large shift to Twitter analysis with {\sc Sentiment-140}, where \name closely estimates a significant absolute performance drop of upto $23\%$ accuracy.
Overall our results here and in Section~\ref{sec:underspec-noisyslices} show early promise that simple, task-agnostic slices that rely on model entropy can be quite effective. 

\begin{table}[t!]

\begin{sc}
\centering
\begin{tabular}{@{}lcc@{}}
  \toprule
  { Source $\rightarrow$ Target Shift} & { Avg. Error} & {Max. Error}\\
  \midrule
  \imdb $\rightarrow$ Counterf. \imdb~\citep{kaushik2020learning}
  & $3.1\% \pm 1.4\%$
  & $4.6\%$
  \\
  \imdb $\rightarrow$ Sentiment 140~\citep{Sentiment140}         
  & $4.7\% \pm 0.8\%$
  & $5.6\%$
  \\
  \imdb $\rightarrow$ Yelp Polarity~\citep{zhangCharacterlevelConvolutionalNetworks2015} 
  & $3.8\% \pm 1.2\%$
  & $4.9\%$
  \\
  \imdb $\rightarrow$ Amazon Polarity~\citep{zhangCharacterlevelConvolutionalNetworks2015}        
  & $0.2\% \pm 0.1\%$
  & $0.3\%$
  \\ 
  \bottomrule
\end{tabular}
\caption[\name{} Target Accuracy Average and Maximum Estimation Error on \imdb]{Average and maximum estimation errors for target accuracy across $3$ models on \imdb, with $95\%$ confidence intervals.}
\label{tab:imdb-sentiment}
\end{sc}
\end{table}



\section{Conclusion}
We introduced \name, a framework for evaluating models under distribution shift that utilizes user-specified slicing functions to reweight estimates. When these slicing functions adequately capture the distribution shift, \name can outperform standard IW by addressing issues of support shift and complex, high-dimensional features. We hope that our framework inspires future work on designing and understanding slices and sets the stage for a new paradigm of model evaluation. 


\section*{Acknowledgements}

We gratefully acknowledge the support of NIH under No. U54EB020405 (Mobilize), NSF under Nos. CCF1763315 (Beyond Sparsity), CCF1563078 (Volume to Velocity), and 1937301 (RTML); ONR under No. N000141712266 (Unifying Weak Supervision); the Moore Foundation, NXP, Xilinx, LETI-CEA, Intel, IBM, Microsoft, NEC, Toshiba, TSMC, ARM, Hitachi, BASF, Accenture, Ericsson, Qualcomm, Analog Devices, the Okawa Foundation, American Family Insurance, Google Cloud, Swiss Re, Total, the HAI-AWS Cloud Credits for Research program, the Stanford Data Science Initiative (SDSI), 
and members of the Stanford DAWN project: Facebook, Google, and VMWare. The Mobilize Center is a Biomedical Technology Resource Center, funded by the NIH National Institute of Biomedical Imaging and Bioengineering through Grant P41EB027060. The U.S. Government is authorized to reproduce and distribute reprints for Governmental purposes notwithstanding any copyright notation thereon. Any opinions, findings, and conclusions or recommendations expressed in this material are those of the authors and do not necessarily reflect the views, policies, or endorsements, either expressed or implied, of NIH, ONR, or the U.S. Government.

\bibliographystyle{plainnat}
\bibliography{mandoline}

\clearpage
\appendix

\onecolumn

\section{Notation Glossary} \label{sec:gloss}

\begin{table*}[!ht]
\centering
\small
\begin{tabular}{l l}
\toprule
Symbol & Used for \\
\midrule
$X$ & Input covariates $X\in\mathcal{X}$.\\
$Y$ & Label $Y\in\mathcal{Y}$.\\
$\p_s$ & Source distribution of $(X, Y)$ with density $p_s$ and expectation $\mathbb{E}_s$. \\
$\p_t$ & Target distribution of $(X, Y)$ with density $p_t$ and expectation $\mathbb{E}_t$. \\
$\D_s$ & Labeled ``validation'' source dataset $\{(x_i^s, y_i^s)\}_{i = 1}^{n_s}$ of size $n_s$ drawn IID from $\p_s$. \\
$\D_t$ & Unlabeled target dataset $\{x_i^t \}_{i = 1}^{n_t}$ of size $n_t$ drawn IID from $\p_t$. \\
$f_\theta$ & Fixed model $f_\theta: \X \ra \Y$ with $\theta$ independent of $\D_s, \D_t$.\\
$\loss$ & Function $\loss: \X \times \Y \ra \R$ that evaluates performance of $f_\theta$. \\
$\Lt$ & $\Lt = \E{t}{\loss(X, Y)}$, the performance of $f_\theta$ on the target population $\p_t$ to be estimated. \\
$g(X)$ & Properties of $X$ that undergo distribution shift, are relevant to the learning task, \\
& and are identified by the user. \\
$h(X)$ & Properties of $X$ that shift, are relevant, but are not identified by the user. \\
$a(X)$ & Properties of $X$ that shift but are irrelevant to the learning task. \\
$b(X)$ & Properties of $X$ that do not undergo distribution shift. \\
$\gtilde(X)$ & Slicing functions $\gtilde(X) = \{\gtilde_1(X), \dots \gtilde_k(X) \}$ where each $\gtilde_i: \X \ra \{-1, 1\}$ \\
& noisily captures $g_i(X)$. \\
$w(X)$ & $w(X) = \frac{p_t(g(X))}{p_s(g(X))}$, weighting based on density ratio of $g$.\\
$\Lg$ & $\Lg = \E{s}{w(X) \loss(X, Y)}$, approximation of $\Lt$ reweighted using $g$, \\
$\D_{s_1}, \D_{s_2}$ & Partition of $\D_s$, where the former $n_{s_1}$ samples are used for learning $w(X)$ and \\
& the latter $n_{s_2}$ samples are used for evaluating $\Lg$ empirically. \\
$\hat{w}(X)$ & Estimated weight function using Algorithm \ref{alg:noise_kliep} on $\D_t$ and $\D_{s_1}$. \\
$\Lghat$ & $\Lghat = \frac{1}{n_{s_2}} \sum_{i = 1}^{n_{s_2}} \hat{w}(x_i^{s_2}) \loss(x_i^{s_2}, y_i^{s_2})$, estimate of $\Lg$. \\
$G$ & Dependency graph $G=(\gtilde,E)$ over the slicing functions. \\
$\phi(g)$ & Vector of potentials on $g$ in \eqref{eq:model_g} (singleton on each $g_i$ and pairwise on each edge in $E$).\\
$\psi_s, \psi_t$ & Canonical parameters in \eqref{eq:model_g} corresponding to $\p_s, \p_t$ respectively. \\
$\delta$ & Difference in canonical parameters, i.e. $\psi_t - \psi_s$. \\ 
$\sigma_s^i, \sigma_t^i$ & Correction matrix approximating the difference in $\gtilde_i$ and $g_i$ for $\p_s, \p_t$, \\
& i.e. $\sigma_s^i(\alpha, \beta) \approx p_s(g_i = \alpha | \gtilde_i = \beta) \, \forall \, \alpha, \beta \in \{-1, 1\}$ and similarly for $\sigma_t^i$. \\
$\mathbb{E}^{\sigma}$ & Conditional ``expectation'' using correction matrices, i.e. \\
& $\Enoisy{s}{r(g) | \gtilde} = \int r(g) \prod_{i = 1}^k \sigma_s^i(g_i, \gtilde_i) dg$ for any function $r(g)$. \\
$\bar{g}$ & $g$ estimated from $\gtilde$ and correction matrices, i.e. $p_s(\bar{g}) = \int \prod_{i = 1}^k \sigma_s^i(\bar{g}_i, \gtilde_i) p_s(\gtilde) d \gtilde$ \\
$\Lgbar$ & $\Lgbar = \E{s}{\frac{p_t(\bar{g}(X))}{p_s(\bar{g}(X))} \loss(X, Y)}$, approximation of $\Lg$ using noise-corrected $\bar{g}$. \\ 
$k'$ & Number of slicing functions that the user fails to correct for. \\
$\eta^{\max}(i), \eta^{\min}(i)$ & Upper and lower bounds on the relative error of $\sigma^i$, e.g. $\big| \frac{p_s(g_i | \gtilde_i) - \sigma_s^i(g_i, \gtilde_i)}{p_s(g_i | \gtilde_i)} \big|$ \\
& for $\eta_s^{\max}(i)$ and $\eta_s^{\min}(i)$, and similarly for $\eta_t$. \\
$r$ & Ratio of relative errors of correction matrices for $\p_t, \p_s$, i.e. $r = \prod_{i = 1}^{k'} \frac{1 + \eta_t^{\max}(i)}{1 - \eta_s^{\min}(i)}$. \\
$M$ & Upper bound on $w(X)$, i.e. $M = \sup_X \frac{p_t(g(X))}{p_s(g(X))}$. \\
\bottomrule
\end{tabular}
\caption{
	Glossary of variables and symbols used in this paper.
}
\label{table:glossary}
\end{table*}
\clearpage

\section{Theoretical Results} \label{app:theory}
We present additional details about the graphical model and algorithm. Then, we provide proofs of Proposition \ref{prop:ghij} and Theorem \ref{thm:main}.

\subsection{Additional Algorithmic and Modeling Details} \label{app:alg}

\subsubsection{Marginalization of Graphical Model}

We demonstrate how the joint density $p(g, \gtilde)$ in $\eqref{eq:model_joint}$ begets $p(g)$ as $\eqref{eq:model_g}$. We can factorize $p(g, \gtilde)$ based on if $\gtilde_i$ has an edge to another $\gtilde_j$ or not:
\begin{align}
    p(g, \gtilde; \theta) &= \frac{1}{Z_\theta} \prod_{i \notin E} \exp (\theta_i g_i + \theta_{ii} g_i \gtilde_i ) \prod_{(i, j) \in E} \exp (\theta_i g_i + \theta_j g_j + \theta_{ii} g_i \gtilde_i + \theta_{jj} g_j \gtilde_j + \theta_{ij} \gtilde_i \gtilde_j ).
\end{align}

Since each $\gtilde_i$ corresponds to one $g_i$, we can also factorize $p(g)$ similarly as 
\begin{align}
    p(g; \psi) &= \frac{1}{Z} \prod_{i \notin E} \exp(\psi_i g_i) \prod_{(i, j) \in E} \exp(\psi_i g_i + \psi_j g_j + \psi_{ij} g_i g_j).
\end{align}

We want to show that there exists $\psi$ such that $p(g; \psi) = \sum_{\gtilde} p(g, \gtilde; \theta)$. Due to the similar factorizations of the distributions, this is equivalent to showing that $\sum_{\gtilde_i \in \{-1, 1 \}} \exp(\theta_i g_i + \theta_{ii} g_i \gtilde_i) \propto \exp(\psi_i g_i)$ for each $i \notin E$, and $\sum_{\gtilde_i, \gtilde_j \in \{-1, 1\}} \exp(\theta_i g_i + \theta_j g_j + \theta_{ii} g_i \gtilde_i + \theta_{jj} g_j \gtilde_j + \theta_{ij} \gtilde_i \gtilde_j) \propto \exp(\psi_i g_i + \psi_j g_j + \psi_{ij} g_i g_j)$ for each $(i, j) \in E$. Note that proportionality must be the same across different values that $g_i$ can take on.

For the case of $i \notin E$,  we must show $\sum_{\gtilde_i \in \{-1, 1 \}} \exp(\theta_i g_i + \theta_{ii} g_i \gtilde_i) = \exp((\theta_i + \theta_{ii})g_i) + \exp((\theta_i -\theta_{ii}) g_i) \propto \exp(\psi_i g_i)$. Setting $g_i = 1$ and $g_i = -1$ and dividing them, we get $\exp(2 \psi_i) = \frac{\exp(\theta_i + \theta_{ii}) + \exp(\theta_i - \theta_{ii})}{\exp(-\theta_i - \theta_{ii}) + \exp(-\theta_{ii} + \theta_i)}$, proving the existence of such $\psi_i$ (note that division allows us to ignore the log-partition functions $Z, Z_\theta$).

For the case of $(i, j) \in E$, set $\psi_i = \theta_i, \psi_j = \theta_j$. We must show there exists a $\psi_{ij}$ that $\sum_{\gtilde_i, \gtilde_j \in \{-1, 1\}} \exp(\theta_{ii} g_i \gtilde_i + \theta_{jj} g_j \gtilde_j + \theta_{ij} \gtilde_i \gtilde_j) \propto \exp(\psi_{ij} g_i g_j)$ for any $g_i, g_j$. Note that plugging in $g_i = g_j = 1$ and $g_i = g_j = -1$ both result in $\exp(\psi_{ij})$ and hence need to yield the same expression on the left hand side, which can be verified (thus more complex parametrizations of $p(g, \gtilde)$ often cannot produce a simple marginal density of $p(g)$). The same observation holds for $g_i = 1, g_j = -1$ and $g_i = -1, g_j = 1$. Setting $g_i = 1, g_j = 1$ and $g_i = 1, g_j = -1$ and dividing them, we can again get a unique expression for $\exp(2\psi_{ij})$ in terms of $\theta_{ii}, \theta_{jj},$ and $\theta_{ij}$.


\subsubsection{Extension to ``incomplete'' slices}

Suppose that slicing functions have an option to abstain when they are unconfident or unapplicable to a data point. We expand the support of each $\gtilde_i$ to $\{-1, 0, 1\}$ and represent this incompleteness as $\gtilde_i(X) = 0$. Fortunately, this is simple to model - we can add potentials to $p(g, \gtilde; \theta)$ to represent this:
\begin{align}
    p(g, \gtilde; \theta) = \frac{1}{Z_\theta} \exp \bigg(\sum_{i = 1}^k \theta_i g_i + \sum_{i = 1}^k \theta_{ii} g_i \gtilde_i + \sum_{i = 1}^k \theta_{i, 0} \ind{\gtilde_i = 0} + \sum_{(i, j) \in E} \theta_{ij} \gtilde_i \gtilde_j\bigg).
\end{align}

Note that $p(g, \gtilde; \theta)$ still yields a marginal density on $g$ of the form $p(g; \psi)$. We would instead get $\exp(2 \psi_i) = \frac{\exp(\theta_i + \theta_{ii}) + \exp(\theta_i - \theta_{ii}) + \exp(\theta_i + \theta_{i, 0})}{\exp(-\theta_i - \theta_{ii}) + \exp(-\theta_{ii} + \theta_i) + \exp(-\theta_i + \theta_{i, 0})}$. We can compute a similar expression for $\exp(2\psi_{ij})$ since this additional potential corresponding to the abstain does not impact the symmetry of the distributions for $g_i = \pm 1$. The remainder of the modeling and Algorithm \ref{alg:noise_kliep} are not affected besides now defining the correction matrices $\sigma_s^i, \sigma_t^i$ to be $2 \times 3$.

\subsubsection{Computational details of Algorithm \ref{alg:noise_kliep}}

We discuss the computational costs of of Algorithm \ref{alg:noise_kliep}, first focusing on the individual expressions $\Enoisy{t}{\delta^\top \phi(g) | \gtilde(x_i^t)}$ and $\Enoisy{s}{\exp(\delta^\top \phi(g)) | \gtilde(x_j^{s_1})}$. Naively, each of these expressions can be evaluated by summing over $2^k$ configurations of $p(g | \gtilde)$ for fixed $\gtilde$. However, due to the factorization of $p(g, \gtilde)$ and $p(g)$, the amount of computation is linear in $k$. Denote $\delta(i)$ as the difference in canonical parameters corresponding to the potential on $g_i$ where $i \notin E$ and $\delta(i, j)$ as a vector of differences corresponding to potentials on $(i, j) \in E$, which we define as $\phi_{ij}(g) = [g_i, g_j, g_i g_j]$ in \eqref{eq:model_g}. Abbreviating $\gtilde(x_i^t)$ as $\gtilde$, we can write $\Enoisy{t}{\delta^\top \phi(g) | \gtilde}$ as
\begin{align*}
    \Enoisy{t}{\delta^\top \phi(g) | \gtilde} &= \int \delta^\top \phi(g) \prod_{i = 1}^k \sigma_t^i(g_i, \gtilde_i) dg \nonumber \\
    &= \prod_{(i, j) \in E} \int \delta(i,j)^\top \phi_{ij}(g) \sigma_t^i(g_i, \gtilde_i) \sigma_t^j(g_j, \gtilde_j) d g_i dg_j \prod_{k \notin E}\int \delta(k) g_k \sigma_t^k(g_k, \gtilde_k) dg_k.
\end{align*}
The number of additions this requires is $4|E| + 2(k - 2|E|) = 2k$, in comparison to $2^k$. Similarly, $\Enoisy{s}{\exp(\delta^\top \phi(g)) | \gtilde} = \prod_{(i, j) \in E} \int \exp(\delta(i, j)^\top \phi_{ij}(g)) \sigma_s^i(g_i, \gtilde_i) \sigma_s^j(g_j, \gtilde_j) dg_i dg_j \times \\  \prod_{k \notin E} \int \exp(\delta(k) g_k) \sigma_s^k(g_k, \gtilde_k) d g_k$, which also requires $2k$ additions. Therefore, evaluating $\eqref{eq:obj_hat}$ has a linear dependency on the number of slicing functions.

We also note that \eqref{eq:obj_hat} has the same computational benefits as LL-KLIEP in that only one pass is needed over the target dataset. The gradient of $\hatfkliep$ is
\begin{align*}
    \frac{\partial \hatfkliep}{\partial \delta} = \frac{1}{n_t} \sum_{i = 1}^n \Enoisy{t}{\phi(g) | \gtilde(x_i^t)} - \frac{\sum_{j = 1}^{n_{s_1}} \Enoisy{s}{\exp(\delta^\top \phi(g)) \phi(g) | \gtilde(x_j^{s_1})}}{\sum_{j = 1}^{n_{s_1}} \Enoisy{s}{\exp(\delta^\top \phi(g)) | \gtilde(x_j^{s_1})}}.
\end{align*}

The first term of the gradient is independent of $\delta$, which means that only one pass is needed on the target dataset even for iterative optimization algorithms to maximize $\hatfkliep$.  

\subsection{Proof of Proposition \ref{prop:ghij}}

We drop the $X$ in $g(X), h(X), a(X), b(X)$ for ease of notation. First, we use Assumption \ref{assump:mand}.\ref{assump:rep}, the chain rule, and Assumption \ref{assump:mand}.\ref{assump:rel}'s conditional independence of $a$ and $Y$:
\begin{align*}
    \E{X, Y \sim \p_s}{\frac{p_t(g(X), h(X))}{p_s(g(X), h(X))} \loss(X, Y) } &= \int \frac{p_t(g, h)}{p_s(g, h)} p_s(x, y) \loss(x, y) \dx \dy \nonumber \\
    &= \int \frac{p_t(g, h)}{p_s(g, h)} p_s(y, g, h, a, b) \loss(x, y) \dx \dy \nonumber \\
    &= \int p_t(g, h) p_s(y, a, b | g, h) \loss(x, y) \dx \dy \nonumber \\
    &= \int p_t(g, h) p_s(y, a | g, h, b) p_s(b | g, h) \loss(x, y) \dx \dy \nonumber \\
    &= \int p_t(g, h) p_s(y | g, h, b) p_s(a | g, h, b) p_s(b | g, h) \loss(x, y) \dx \dy. 
\end{align*}

Using Assumption \ref{assump:mand}.\ref{assump:rel}'s conditional independence of $a$ and $Y$, \ref{assump:mand}.\ref{assump:rep}, and the fact that there is no concept drift (i.e., $p_s(Y|X)= p_t(Y|X)$), we get that $p_s(y | g, h, b) = p_t(y | g, h, b)$. Next, by Assumption \ref{assump:mand}.\ref{assump:rel}'s conditional independence of $a$ and $\loss(X, Y)$, we can integrate out $p_s(a | g, h, b)$ to get
\begin{align*}
   \E{X, Y \sim \p_s}{\frac{p_t(g(X), h(X))}{p_s(g(X), h(X))} \loss(X, Y) } = \int p_t(g, h) p_t(y | g, h, b) p_s(b | g, h) \loss(x, y) \dx \dy.
\end{align*}

\noindent By Assumption \ref{assump:mand}.\ref{assump:rel}'s conditional independence of $a$ and $b$, we have that $p_s(b | g, h) = p_s(b | g, h, a)$. By Assumption \ref{assump:mand}.\ref{assump:shift} and \ref{assump:mand}.\ref{assump:rep}, this is equal to $p_t(b | g, h, a) = p_t(b | g, h)$. We thus have
\begin{align*}
    \E{X, Y \sim \p_s}{\frac{p_t(g(X), h(X))}{p_s(g(X), h(X))} \loss(X, Y) } &= \int p_t(g, h) p_t(y | g, h, b) p_t(b | g, h) \loss(x, y) \dx \dy \nonumber \\
    &= \int p_t(y, g, h, b) \loss(x, y) \dx \dy.
\end{align*}

Using Assumption \ref{assump:mand}.\ref{assump:rel}'s conditional independence of $a$ and $\loss(X, Y)$ and Assumption \ref{assump:mand}.\ref{assump:rep} again, this is equivalent to
\begin{align*}
    \E{X, Y \sim \p_s}{\frac{p_t(g(X), h(X))}{p_s(g(X), h(X))} \loss(X, Y) } &= \int p_t(y, g, h, a, b) \loss(x, y) \dx \dy = \int p_t(x, y) \loss(x, y) \dx \dy \nonumber \\
    &= \mathbb{E}_{X, Y \sim \p_t}[\loss(X, Y)].
\end{align*}

\subsection{Proof of Theorem \ref{thm:main}}
\label{mand:thmmain_prf}
Recall that $|\Lt - \Lghat|$ can be decomposed into the sum of $|\Lt - \Lg|$, $|\Lg - \Lgbar|$, $|\Lgbar - \mathbb{E}_s[\Lghat]|$, and $|\mathbb{E}_s[\Lghat] - \Lghat|$. We bound each difference individually. Assume without loss of generality that $\loss(x, y) \le 1$.

\begin{lemma}
We abbreviate $\mathcal{P}(h(X)|g(X))$ as $\mathcal{P}(h|g)$. Then,
\begin{align*}
    |\Lt - \Lg| \le 
     \I{}{} p_t(g(x)) \cdot |p_t(h(x)|g(x)) - p_s(h(x)|g(x))| \, \dx 
\end{align*}

Note that if $p_s(h|g) = p_t(h|g)$ (or $h(X)$ is empty), then weighting based on $g(X)$ gives us an unbiased estimate of $\Lt$.
\end{lemma}

\begin{proof}
$\Lg$ is equal to $\E{s}{\frac{p_t(g(X))}{p_s(g(X))} \loss(X, Y)}$. Therefore, using our result from Proposition \ref{prop:ghij},
\begin{align*}
    |\Lt - \Lg| &= \bigg|\E{s}{ \bigg(\frac{p_t(X)}{p_s(X)} - \frac{p_t(g(X))}{p_s(g(X))} \bigg) \loss(X, Y)} \bigg| = \bigg|\int p_s(x, y) \left(\frac{p_t(g, h)}{p_s(g, h)} - \frac{p_t(g)}{p_s(g)} \right)  \loss(x, y) \dx \dy \bigg| \nonumber \\
    &= \bigg| \int \frac{p_t(g)}{p_s(g)} p_s(x, y) \bigg(\frac{p_t(h|g)}{p_s(h|g)} - 1 \bigg) \loss(x, y) \dx \dy \bigg| \nonumber \\
    &= \bigg| \int \frac{p_t(g)}{p_s(g)} \frac{p_s(x, y)}{p_s(h|g)} (p_t(h|g) - p_s(h|g)) \loss(x, y) \dx \dy \bigg|.  
\end{align*}

\noindent Note that $\frac{p_t(g)}{p_s(g)} \frac{p_s(x, y)}{p_s(h|g)}$ can be simplified into $p_t(g) \frac{p_s(x, y)}{p_s(g, h)} = p_t(g) \frac{p_s(x, y, g, h)}{p_s(g, h)} =  p_t(g) p_s(x, y | g, h)$. Our bound is now
\begin{align*}
    |\Lt - \Lg| &= \bigg| \int p_t(g) p_s(x, y | g, h) \cdot (p_t(h|g) - p_s(h|g)) \loss(x, y) \dx \dy \bigg| \nonumber \\
    &\le \int p_t(g) p_s(x, y | g, h) \cdot |p_t(h|g) - p_s(h|g)| \loss(x, y) \dx \dy.
\end{align*}

\noindent We now use the fact that $0 \le \loss(x, y) \le 1$ (by assumption) and $0 \le p_s(x, y | g, h) \le 1$:
\begin{align*}
    |\Lt - \E{s}{\Lg}| \le \int p_t(g) \cdot |p_t(h|g) - p_s(h|g) | \,\dx. 
\end{align*}

\end{proof}

\begin{lemma} Without loss of generality, suppose the user fails to correct on the first $k'$ slicing functions. The bias in estimation error due to incorrect noisy matrices is
\begin{align*}
    |\Lg - \Lgbar| \le 2 \sum_{i = 1}^{k'} \bigg(\frac{\eta_t^{\max}(i)}{1 - \eta_t^{\min}(i)} + \frac{\eta_s^{\max}(i)}{1 - \eta_s^{\min}(i)}\bigg).
\end{align*}
\end{lemma}

\begin{proof}
We can write the difference $|\Lg - \Lgbar|$ as $\big|\mathbb{E}_{s}\big[\big(\frac{p_t(g(X))}{p_s(g(X))} - \frac{p_t(\bar{g}(X))}{p_s(\bar{g}(X))} \big) \loss(X, Y)\big] \big|$. Partition $\X$ into $\X^+$ and $\X^-$, where $\X^+ = \big\{X: \frac{p_t(g(X))}{p_s(g(X))} \ge \frac{p_t(\bar{g}(X))}{p_s(\bar{g}(X))}\big\}$ and $\X^- = (\X^+)^C$. Then using the fact that $\loss(x, y) \le 1$, the difference can be written as
\begin{align}
    |\Lg - \Lgbar| &\le \int_{\X^+} p_s(x, y) \frac{p_t(\bar{g}(x))}{p_s(\bar{g}(x))} \bigg(\frac{p_t(g(x))}{p_t(\bar{g}(x))} \cdot \frac{p_s(\bar{g}(x))}{p_s(g(x))} - 1 \bigg) \dx \dy \nonumber \\
    &+ \int_{\X^-} p_s(x, y) \frac{p_t(g(x))}{p_s(g(x))} \bigg(\frac{p_t(\bar{g}(x))}{p_t(g(x))} \cdot \frac{p_s(g(x))}{p_s(\bar{g}(x))} - 1\bigg) \dx \dy.\label{eq:noise_correction}
\end{align}

We now bound the ratio of $g$ to $\bar{g}$ in $\p_s$ and $\p_t$. $\frac{p_t(g(x))}{p_t(\bar{g}(x))}$ is equal to $\frac{\sum_{\gtilde} p_t(g(x) | \gtilde) p_t(\gtilde)}{\sum_{\gtilde} p_t(\bar{g}(x) | \gtilde) p_t(\gtilde)}$, and using the log sum inequality,
\begin{align*}
    \frac{p_t(g(x))}{p_t(\bar{g}(x))}
    &\le
    \exp \bigg(\frac{1}{p_t(g(x))} \sum_{\gtilde} p_t(g(x), \gtilde) \log \frac{p_t(g(x) | \gtilde)}{p_t(\bar{g}(x) | \gtilde)} \bigg) \nonumber
    \\
    &=
    \prod_{\gtilde} \exp \bigg(p_t(\gtilde | g(x))  \sum_{i = 1}^{k'} \log \frac{p_t(g_i(x) | \gtilde_i)}{p_t(\bar{g}_i(x) | \gtilde_i)} \bigg) \nonumber
    \\
    &= \prod_{i = 1}^{k'} \prod_{\gtilde}  \bigg(\frac{p_t(g_i(x) | \gtilde_i)}{p_t(\bar{g}_i(x) | \gtilde_i)} \bigg)^{{p_t(\gtilde | g(x))} } \nonumber
    \\
    &=
    \prod_{i = 1}^{k'} \bigg(\frac{p_t(g_i(x) | \gtilde_i = 1)}{p_t(\bar{g}_i(x) | \gtilde_i = 1)} \bigg)^{p_t(\gtilde_i = 1 | g(x))} \bigg(\frac{p_t(g_i(x) | \gtilde_i = -1)}{p_t(\bar{g}_i(x) | \gtilde_i = -1)} \bigg)^{p_t(\gtilde = -1 | g(x))} \nonumber
    \\
    &\le
    \prod_{i = 1}^{k'} \max \bigg\{ \frac{p_t(g_i(x) | \gtilde_i = 1)}{p_t(\bar{g}_i(x) | \gtilde_i = 1)}, \frac{p_t(g_i(x) | \gtilde_i = -1)}{p_t(\bar{g}_i(x) | \gtilde_i = -1)} \bigg\}.
\end{align*}

\noindent Let the final expression above be $\prod_{i = 1}^{k'} a_i(x)$. Similarly, 
\begin{align*}
\frac{p_t(\bar{g}(x))}{p_t(g(x))} \le \prod_{i = 1}^{k'} \max \bigg\{ \frac{p_t(\bar{g}_i(x) | \gtilde_i = 1)}{p_t(g_i(x) | \gtilde_i = 1)},  \frac{p_t(\bar{g}_i(x) | \gtilde_i = -1)}{p_t(g_i(x) | \gtilde_i = -1)} \bigg\}. \nonumber
\end{align*}
\noindent and let this final expression be $\prod_{i = 1}^{k'} b_i(x)$. We get similar bounds for $\p_s$, where we let $\frac{p_s(g(x))}{p_s(\bar{g}(x))} \le \prod_{i = 1}^{k'} c_i(x)$ and $\frac{p_s(\bar{g}(x))}{p_s(g(x))} \le \prod_{i = 1}^{k'} d_i(x)$. Plugging these back into \eqref{eq:noise_correction},
\begin{align}
    |\Lg - \Lgbar| &\le \int_{\X^+} \lt p_s(x, y) \frac{p_t(\bar{g}(x))}{p_s(\bar{g}(x))} \prod_{i = 1}^{k'} a_i(x) d_i(x) - 1 \rt  \dx \dy \notag \\
    &+ \int_{\X^-} \lt p_s(x, y) \frac{p_t(g(x))}{p_s(g(x))} \prod_{i = 1}^{k'} b_i(x) c_i(x) - 1 \rt \dx \dy.
    \label{eq:noise_correction2}
\end{align}

Via a telescoping argument, we can show that $\prod_{i = 1}^{k'} a_i(x) d_i(x) - 1 \le \sum_{i = 1}^{k'} |a_i(x) - 1| + |d_i(x) - 1|$ and $\prod_{i = 1}^{k'} b_i(x) c_i(x) - 1 \le \sum_{i = 1}^{k'} |b_i(x) - 1| + |c_i(x) - 1|$. Then,
\begin{align*}
    |a_i(x) - 1| &\le \max\bigg\{\Big|\frac{p_t(g_i(x) | \gtilde_i = 1) - p_t(\bar{g}_i(x) | \gtilde_i = 1)}{p_t(\bar{g}_i(x) = 1 | \gtilde_i = 1)}\Big|, \Big| \frac{p_t(g_i(x) | \gtilde_i = -1) - p_t(\bar{g}_i(x) | \gtilde_i = -1)}{p_t(\bar{g}_i(x) | \gtilde_i = -1)}\Big| \bigg\} \nonumber  \\
    &\le \frac{\eta_t^{\max}(i)}{1 - \eta_t^{\min}(i)}.
\end{align*}

\noindent Similarly, $|b_i(x) - 1|$ is also at most $\frac{\eta_t^{\max}(i)}{1 - \eta_t^{\min}(i)}$, and $|c_i(x) - 1|, |d_i(x) - 1| \le \frac{\eta_s^{\max} (i)}{1 - \eta_s^{\min} (i)}$. Therefore, \eqref{eq:noise_correction2} becomes
\begin{align*}
    |\Lg - \Lghat| \le \sum_{i = 1}^{k'} \bigg(\frac{\eta_t^{\max}(i)}{1 - \eta_t^{\min}(i)} + \frac{\eta_s^{\max}(i)}{1 - \eta_s^{\min}(i)}\bigg) \bigg(\int_{\X^+} p_s(x, y) \frac{p_t(\bar{g}(x))}{p_s(\bar{g}(x))} \dx \dy + \int_{\X^-} p_s(x, y) \frac{p_t(g(x))}{p_s(g(x))} \dx \dy \bigg).
\end{align*}

\noindent Finally, we can bound $\int_{\X^+} p_s(x, y) \frac{p_t(\bar{g}(x))}{p_s(\bar{g}(x))} \dx \dy + \int_{\X^-} p_s(x, y) \frac{p_t(g(x))}{p_s(g(x))} \dx \dy \le \\
 \E{s}{\max \{w(\gbar(X)), w(g(X)) \}} \le rM$, using Lemma \ref{lemma:rm}. Our final bound is thus
\begin{align*}
    |\Lg - \Lghat| \le rM \sum_{i = 1}^{k'} \bigg(\frac{\eta_t^{\max}(i)}{1 - \eta_t^{\min}(i)} + \frac{\eta_s^{\max}(i)}{1 - \eta_s^{\min}(i)}\bigg).
\end{align*}
\end{proof}

\begin{lemma} 
Define $\Mhat = \sup_X \hat{w}(g(X))$. Setting $n_{s_1} = \frac{n_s}{2}$, then with probability at least $1 - \alpha$,
\begin{align*}
    |\Lgbar - \mathbb{E}_s[\Lghat]| \le c_{s, \loss}  \Mhat \sqrt{\frac{\log(2/\alpha)}{n_s}},
\end{align*}

where $c_{s, \loss}$ is a constant depending on properties of $\p_s$ and $\loss(x, y)$.
\end{lemma}

\begin{proof}
$|\Lgbar - \mathbb{E}_s[\Lghat]|$ can be written as $|\E{s}{(w(\bar{g}(X)) - \hat{w}(X)) \loss(X, Y)} |$, where $\hat{w}(X)$ is constructed from Algorithm \ref{alg:noise_kliep}. We bound this by $c_{s, \loss} \E{s}{w(\bar{g}(X)) - \hat{w}(X)}$, where $c_{s, \loss}$ is a constant dependent on $\p_s$ and how the loss function changes across $x, y$. $\E{s}{w(\bar{g}(X))} = 1$, and recall that $\hat{w}(X) = \Enoisy{s}{\hat{w}(g) | \gtilde(X)} = \frac{\Enoisy{s}{\exp(\dhat^\top \phi(g)) | \gtilde(X)}}{\frac{1}{n_{s_1}} \sum_{j = 1}^{n_{s_1}} \Enoisy{s}{\exp(\dhat^\top \phi(g)) | \gtilde(x_j^{s_1})}}$. Then, the difference to bound is equivalent to
\begin{align}
     |\Lgbar - \mathbb{E}_s[\Lghat]| \le \frac{ c_{s, \loss} \big|\frac{1}{n_{s_1}} \sum_{j = 1}^{n_{s_1}} \mathbb{E}^\sigma_{s}[\exp(\dhat^\top \phi(g)) | \gtilde(x_j^{s_1})] - \mathbb{E}_{s}[\exp(\dhat^\top \phi(\bar{g}))]\big|}{\frac{1}{n_{s_1}} \sum_{j = 1}^{n_{s_1}} \mathbb{E}^\sigma_{s}[\exp(\dhat^\top \phi(g)) | \gtilde(x_j^{s_1})]}, \label{eq:gbar_diff}
\end{align}
\noindent where $\mathbb{E}_{s}[\mathbb{E}^\sigma_{s}[\exp(\dhat^\top \phi(g)) | \gtilde(X)]] = \mathbb{E}_{s}[\exp(\dhat^\top \phi(\bar{g}(X)))]$ by definition of $p(\bar{g})$ and $\mathbb{E}_s^\sigma$. Define $\epsilon = \\\frac{1}{n_{s_1}} \sum_{j = 1}^{n_{s_1}} \mathbb{E}^\sigma_{s}[\exp(\dhat^\top \phi(g)) | \gtilde(x_j^{s_1})] - \mathbb{E}_{s}[\exp(\dhat^\top \phi(\bar{g}(X)))]$. Our bound becomes
\begin{align}
    |\Lgbar - \mathbb{E}_s[\Lghat]| \le \frac{ c_{s, \loss} |\epsilon|}{\frac{1}{n_{s_1}} \sum_{j = 1}^{n_{s_1}} \mathbb{E}^\sigma_{s}[\exp(\dhat^\top \phi(g)) | \gtilde(x_j^{s_1})]}. \label{eq:gbar_diff2}
\end{align}

Note $\max_{\bar{g}} \dhat^\top \phi(\bar{g}) = \norm{\dhat}_1$ and $\E{}{\epsilon} = 0$, so applying Hoeffding's inequality gives us $|\epsilon| \le \exp \lt \norm{\dhat}_1 \rt \sqrt{\frac{\log(2/\alpha)}{2n_{s_1}}}$ with probability at least $1 - \alpha$. Plugging this into \eqref{eq:gbar_diff2},
\begin{align*}
    |\Lgbar - \mathbb{E}_s[\Lghat]| \le \frac{c_{s, \loss} \exp\lt \norm{\dhat}_1\rt \sqrt{\frac{\log(2/\alpha)}{2n_{s_1}}}}{\frac{1}{n_{s_1}} \sum_{j = 1}^{n_{s_1}} \mathbb{E}^\sigma_{s}[\exp(\dhat^\top \phi(g)) | \gtilde(x_j^{s_1})]}.
\end{align*}

Furthermore, $\Mhat =\sup_X \hat{w}(g(X)) = \frac{\exp(\left|\dhat\right|_1)}{\frac{1}{n_{s_1}} \sum_{j = 1}^{n_{s_1}} \mathbb{E}^\sigma_{s}[\exp(\dhat^\top \phi(g)) | \gtilde(x_j^{s_1})]}$. Then, using this definition and the fact that $n_{s_1} = \frac{n_s}{2}$, we arrive at our desired bound,
\begin{align*}
    |\Lgbar - \mathbb{E}_s[\Lghat]| \le c_{s, \loss}  \Mhat \sqrt{\frac{\log(2/\alpha)}{n_s}}.
\end{align*}
\end{proof}

\begin{lemma} 
Setting $n_{s_2} = \frac{n_s}{2}$, then with probability at least $1 - \alpha$,
\begin{align*}
    |\mathbb{E}_s[\Lghat] - \Lghat | \le \Mhat \sqrt{\frac{\log(2 / \alpha)}{n_s}}.
\end{align*}
\end{lemma}

\begin{proof}
This result follows from a standard application of Hoeffding's inequality using the definition of $\Mhat$ as an upper bound on $\hat{w}$.
\end{proof}

\noindent Using the above lemmas and taking a union bound gives us the desired bound for Theorem \ref{thm:main}.

Lastly, to understand $\Mhat$, we discuss two things: first, as $n_s \wedge n_t \ra \infty$, $\Mhat \xrightarrow{p} \sup w(\gbar(X))$. Second, $w(\gbar(X))$ is bounded in terms of $\sup w(X)$ and properties of the correction matrices.

\begin{proposition}
$\Mhat \xrightarrow{p} \sup_X w(\gbar(X))$ as $n_s \wedge n_t \ra \infty$.
\end{proposition}

\begin{proof} (Informal)
Since $w(\bar{g}(X)) = \frac{p_t(\bar{g}(X))}{p_s(\bar{g}(X))}$ and $\gbar$ follows the same graphical model structure as $g$, $w(\gbar(X)) = \frac{\exp(\dbar^\top \phi(\bar{g}(X)))}{\E{s}{\exp(\dbar^\top \phi(\bar{g}))}}$,  where $\dbar = \argmin_{\delta}{ \E{s}{\delta^\top \phi(\gbar)}} - \log \E{s}{\exp(\delta^\top \phi(\gbar))}$, which is also the solution to the population version of \eqref{eq:obj_hat}. \cite{kim2019twosample} show that $\dhat \xrightarrow{p} \dbar$ as $n_s \wedge n_t \ra \infty$ via an argument that \eqref{eq:obj_hat} converges pointwise to a maximum log-likelihood objective function, which yields a consistent estimator. Therefore, $\frac{\exp(\dhat^\top \phi(\bar{g}(X)))}{\frac{1}{n_{s_1}} \sum_{j = 1}^{n_{s_1}} \Enoisy{s}{\exp(\dhat^\top \phi(g)) | \gtilde(x_j^{s_1}) }} \xrightarrow{p} \frac{\exp(\dbar^\top \phi(\gbar(X)))}{\E{s}{\exp(\dbar^\top \phi(\gbar)}}$, which implies convergence of their suprema over $X$.
\end{proof}
 
 \begin{lemma}
 $\sup_X w(\gbar(X)) \le r M$, where $M = \sup_X \frac{p_t(g(X))}{p_s(g(X))}$ and $r = \prod_{i = 1}^{k'} \frac{1 + \eta_t^{\max}(i)}{1 - \eta_s^{\max}(i)}$.
 \label{lemma:rm}
 \end{lemma}
 
 \begin{proof}
 We can write $\frac{p_t(\gbar)}{p_s(\gbar)}$ as $\frac{p_t(g) \cdot \frac{p_t(\gbar)}{p_t(g)}}{p_s(g) \cdot \frac{p_s(\gbar)}{p_s(g)}}$. By the log sum inequality, we have 
 \begin{align}
     \frac{p_t(\gbar)}{p_t(g)} \le \prod_{i = 1}^{k'} \max \bigg\{ \frac{p_t(\bar{g}_i | \gtilde_i = 1)}{p_t(g_i | \gtilde_i = 1)},  \frac{p_t(\bar{g}_i | \gtilde_i = -1)}{p_t(g_i | \gtilde_i = -1)} \bigg\} \le \prod_{i = 1}^{k'} (1 + \eta_t^{\max}(i)), \\
     \frac{p_s(\gbar)}{p_s(g)} \ge \prod_{i = 1}^{k'} \min \bigg\{ \frac{p_s(\bar{g}_i | \gtilde_i = 1)}{p_s(g_i | \gtilde_i = 1)},  \frac{p_s(\bar{g}_i | \gtilde_i = -1)}{p_s(g_i | \gtilde_i = -1)} \bigg\} \ge \prod_{i = 1}^{k'} (1 - \eta_s^{\min}(i)), \label{eq:Mhat_lb}
 \end{align}
 where \eqref{eq:Mhat_lb} comes from from taking the reciprocal of an upper bound on $\frac{p_s(g)}{p_s(\gbar)}$. Then $\frac{p_t(\gbar)}{p_s(\gbar)} \le \frac{p_t(g)}{p_s(g)} \cdot \prod_{i = 1}^{k'} \frac{1 + \eta_t^{\max}(i)}{1 - \eta_s^{\min}(i)} \le rM$.
 \end{proof}
 

\section{Additional Experimental Details}

\subsection{Synthetic Experiments}
\label{app:synthetics}
\subsubsection{Support Shift Experiment}
In this synthetic experiment (described in Section~\ref{sec:synthetics}), we show how classifier-based importance weighting can perform poorly when the source and target distributions have little overlap. We study the case where the classifier is a logistic regression classifier. We note that if the supports are completely disjoint, logistic regression-based IW can actually still work well if self-normalization is used (since the logits output by the classifier will be bounded as long as appropriate regularization is used). However, if the supports are \emph{nearly} disjoint, but there are a very small number of source examples that do fall inside the support of the target distribution, the logistic regression classifier assigns very high weight to these examples compared to the rest, and thus reduce the effective sample size \citep{Owen2013MC}.

We generate data with a single binary $g_1$, and set $\theta_1$ for the source and target datasets such that $P(g_1 = 0) = 0.25$ on the source dataset and 0.75 on the target training set. We set the first component $X$ to be normally distributed conditioned on $g_1$, with mean $2g_1 - 1$ and random variance. We then append a ``spurious'' feature (``$a(X)$'') that is 1 with small probability $p$ for source examples, and 1 for \emph{all} target examples. We generate binary ``labels'' $Y$ for each datapoint following a logistic model with randomly generated coefficient on the first component of $X$. The goal is to estimate the average of $Y$ on the target dataset. We repeat this for different values of $p$, and for multiple random seeds each. We generate 10,000 points for both source and target datasets.

In Section~\ref{sec:synthetics}, we plot the mean absolute difference between the CBIW estimate of $\E{t}{Y}$ and the true value, divided by the mean absolute difference between $\E{s}{Y}$ and $\E{t}{Y}$. For small values of $p$, the CBIW estimate of the target value is actually even worse than simply using the estimate from the source dataset. By contrast, \name~only uses the $g(X)$ representation, so the spurious feature $a(X)$ does not affect its performance.

\subsubsection{High Dimensionality}
\begin{figure}[t!]
    \centering
    \includegraphics[width=0.5\linewidth]{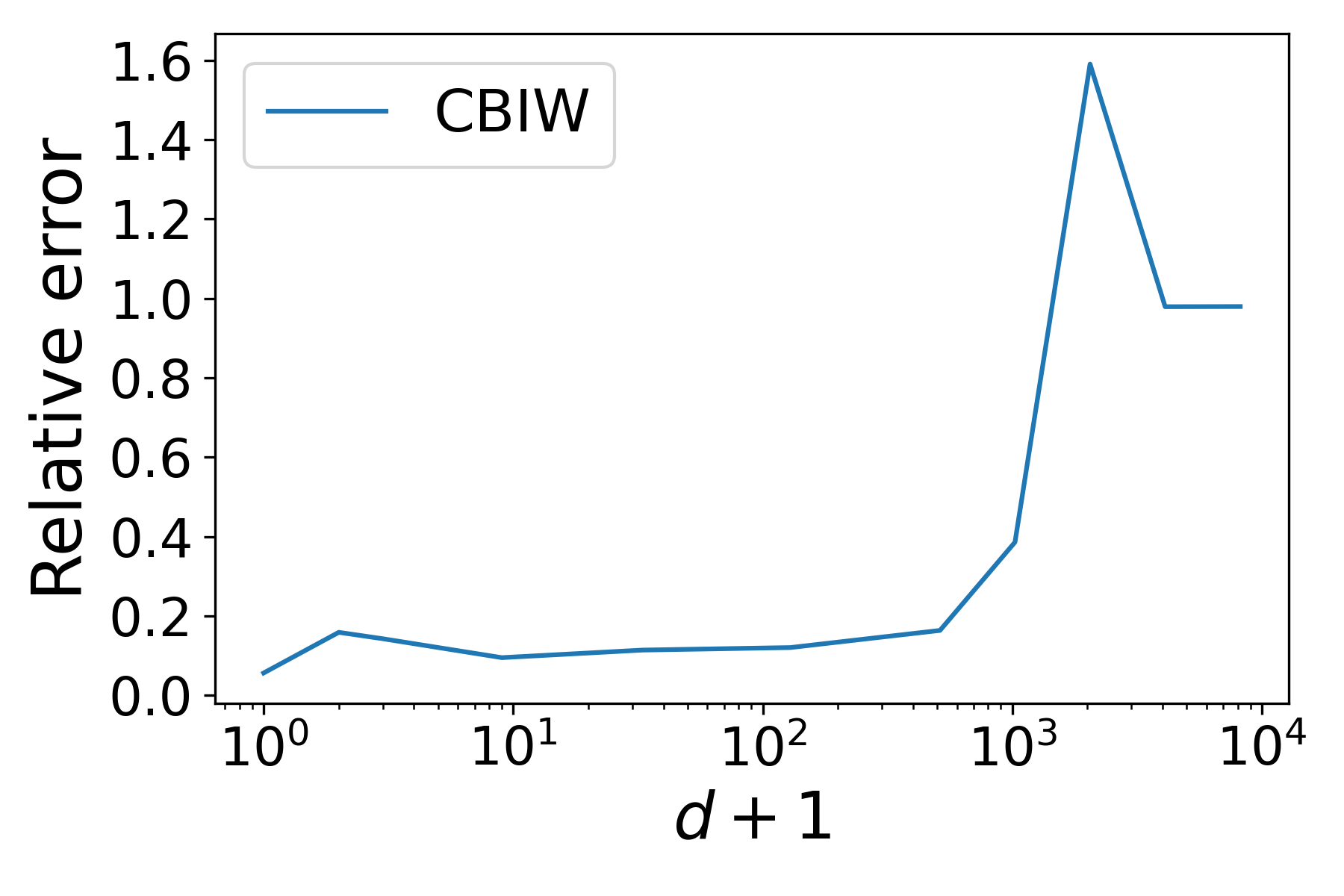}
    \caption[Relative Error of CBIW as Number of Irrelevant Features Increases]{Relative error of CBIW in the presence of $d$ additional irrelevant features.}
    \label{fig:dim-synthetic}
    
\end{figure}

Similarly, standard importance weighting baselines can struggle when the input data is very high-dimensional. As an example, we again generate data with a single binary $g_1$, and set $\theta_1$ for the source and target datasets such that $P(g_1 = 0) = 0.25$ on the source dataset and 0.75 on the target training set. Examples with $g_1 = 0$ are randomly sampled from the circle centered at $(-1, 0)$ with radius 1, and examples with $g_1 = 1$ are randomly sampled from the circle centered at $(0, 1)$ with radius 1. Then, we append $d$ additional ``irrelevant'' features drawn from $\N(0, 25 I_d)$. Intuitively, as the number of added features grows compared to the number of datapoints, a logistic regression classifier trained to distinguish between source and target examples can overfit to ``memorize'' points based on these features. We generate 10,000 points from the target dataset, and only 1,000 from the source dataset.

We generate random binary ``labels'' $Y$ for each datapoint following a logistic model with coefficients $1/\sqrt{2}$ on the first two components of $X$. The goal is again to estimate the average of $Y$ on the target dataset. (The true average is approximately 0.42, versus approximately 0.59 on the source dataset.) We repeat this for different values of $d$ (the number of ``irrelevant features''), and for multiple random seeds each. Results are plotted in Figure~\ref{fig:dim-synthetic}. As before, we plot the mean absolute difference between the CBIW estimate of $\E{t}{Y}$ and the true value, divided by the mean absolute difference between $\E{s}{Y}$ and $\E{t}{Y}$.  When there are only a few irrelevant features, CBIW performs well as it generally ignores them. When the number of irrelevant features is on the same order as the number of source datapoints or more, the estimation error rapidly rises, eventually being as bad or worse as simply using $\E{s}{Y}$ as an estimate of $\E{t}{Y}$. By contrast, the mean absolute difference between the \name~estimate of $\E{t}{Y}$ and the true value, divided by the mean absolute difference between $\E{s}{Y}$ and $\E{t}{Y}$, is only about 0.06 (comparable to that of the CBIW estimate when the number of irrelevant features $d = 0$), since \name\ works with the $g(X)$ representation and thus ignores these irrelevant features.

In many real-world tasks, such as medical image classification, the raw dimensionality of the data exceeds the number of datapoints (sometimes by orders of magnitude). In situations such as these, a classifier might simply ``memorize'' which examples come from the source vs. target distributions, thus obtaining poor importance weights as in this synthetic example.

\subsubsection{Noisy Slices} 

In this experiment (described in Section~\ref{sec:synthetics}), we demonstrate how correcting noisy slices can yield better estimates than when the slices are not corrected. Accurate correction matrices $\sigma_s^i, \sigma_t^i$ for each slice can recover the density ratio of $g$ to reweight with. However, if the slices $\gtilde$ are weakly correlated with the true $g$ and not corrected, the generated weights can be significantly different and result in an inaccurate estimate of $\Lg$.

We consider an example with $k = 2$ and one edge between the two slicing functions and construct datasets of size $n_s = n_t = 100000$ with distributions $p_s(g, \gtilde; \theta_s)$ and $p_t(g, \gtilde; \theta_t)$ according to \eqref{eq:model_joint} with randomly generated $0 \le \theta_s$, $\theta_t \le 1$. \name\ uses Algorithm \ref{alg:noise_kliep} with accurate correction matrices $\sigma_s^i = p_s(g_i | \gtilde_i)$ for each $i$ and for $\p_t$ as well. We compare this to a noise-unaware baseline where we run Algorithm \ref{alg:noise_kliep} with each correction matrix equal to the identity matrix, which entails that the user thinks $\gtilde$ and $g$ are the same. These weights are used to compute an estimate of a simple loss function $\loss(x, y)$ which takes on four values $[0.02, 0.2, 0.9, 0.2]$ depending on the values of $g_1(X), g_2(X) \in \{-1, 1\}$.

The performance of these two methods depends on what the true correlation between $g$ and $\gtilde$ is. We capture this correlation via the canonical parameter $\theta_{ii}$ for each $i$. That is, after randomly setting all other canonical parameters between $0$ and $1$, we uniformly set all $\theta_{ii}$ to one value and vary this between $0$ and $3$ to generate our datasets. When all $\theta_{ii}$ are equal to $0$, $g_i \indep \gtilde_i$ and no information can be extracted from only observing $\gtilde$. As $\theta_{ii}$ gets larger, the correlation between $g_i$ and $\gtilde_i$ increases until they are essentially the same. Our results in Figure \ref{fig:synthetic} confirm that for $\theta_{ii} = 0$, both \name\ and the noise-unaware version do not correct for the distribution shift. This is because $\Lghat$ has uniform weights since no information about $g$ can be extracted from the observable $\gtilde$. As $\theta_{ii}$ increases,  both methods approach relative error equal to $0$, since both $\gtilde$ and $g$ will be sufficiently accurate to reweight based on. However, for intermediate values of $\theta_{ii}$, there is a gap between the performance of \name\ and the noise-unaware approach. In these cases, the correlation between $g$ and $\gtilde$ is enough that learning from $\gtilde$ is meaningful, but still weak enough that noise-correction is important. 

\subsection{Experiments on Real Data}
\label{app:experiments}
We use PyTorch \citep{paszke2019pytorch} for all experiments.

\subsubsection{Baselines}
\label{app:baselines-first}
As described in Section~\ref{sec:experiments}, we compare the performance of \name~at estimating target accuracy compared to the baselines \textsc{CBIW}, \textsc{KMM}, and \textsc{uLSIF} run on neural network features. We also compare to \textsc{CBIW-FT} (CBIW where the entire neural network is fine-tuned) and \textsc{Source} (simply using the accuracy on the source dataset as an estimate of that on the target). Additionally, we can also run \textsc{CBIW}, \textsc{KMM}, and \textsc{uLSIF} on the \emph{slice} representations instead, i.e. swapping out the KLIEP-based stage of \name\ for another importance weighting algorithm. Finally, we also compare to the "simple" baseline of reweighting points based on the ratio of the frequency of their slice in the target dataset to the frequency of their slice in the source dataset. We discuss the baseline methods in more detail in Appendix~\ref{app:baselines}. 

\subsubsection{\celeba}
\para{Datasets.} We use the CelebA dataset\footnote{Available from \url{http://mmlab.ie.cuhk.edu.hk/projects/CelebA.html}} \citep{liu2015faceattributes}, which consists of 202,599 images of celebrity faces annotated with various metadata (e.g., gender, hair color, whether the image is blurry, etc.). We randomly allocate 1/3 of the dataset for training, and split the remainder into validation (``source''), and test (``target'') splits. The latter split was done to induce distribution shift between the source and target datasets: specifically, we allocate 80\% of (non-training) ``blurry'' images to the target dataset, along with enough non-blurry images so that the target is 30\% blurry images overall. The remainder of (non-training) images are allocated to the source dataset; this results in only approximately 1\% of source images being ``blurry''.

\para{Models}. We train ResNet-18 and ResNet-50 models on the training set for 5 epochs, starting from ImageNet-pretrained checkpoints (provided by PyTorch). We train with a learning rate of $0.0002$, weight decay of 0.0001, and batch size of 256 on 4 NVIDIA V100 GPUs (similar to the settings provided in \citep{sagawa2020distributionally}, but with the batch size doubled and correspondingly the learning rate as well). We train three separate models of each type, starting from different random seeds, and evaluate how well \name~and the baseline methods estimate performance of these models on the test (target) set by reweighting the validation (source) set predictions. Trial-by-trial results for each method are given in Table~\ref{tab:celeba-full}.

\para{Methods}. We use the default implementation of \name\ (with no noise correction factor, as we use the true provided metadata). As there is only one $g_i$ in this case (blurriness), there is no need to specify an edge list. For CBIW, we use the scikit-learn LogisticRegression function, with the default regularization strength of 1.0 and the L-BFGS optimizer; we train until convergence to the default tolerance is reached. For CBIW-FT, we train models for 5 epochs with the same hyperparameter settings as above, but on the concatenated source and target datasets, with the label 0 for source images and 1 for target images. For KMM, we use the publicly available implementation provided by \cite{fang2020rethinking}, while for uLSIF we use the \texttt{densratio} package \citep{densratio}; we randomly sample 10,000 examples each from the source and target datasets before running KMM and uLSIF due to the high computational / memory demands of these methods. We also compare to the slice-based baselines discussed in Appendix~\ref{app:baselines-first}. (Note that for CelebA, since there is only one slice under consideration, it can be shown that the ``Simple'' baseline is actually equivalent to \name\ [as reflected in the results].)

\begin{table}[]
    \centering
    \resizebox{\linewidth}{!}{%
        \begin{tabular}{l|c|c|c|c|c|c}
            \toprule
& RN-18 seed 0 & RN-18 seed 1 & RN-18 seed 2 & RN-50 seed 0 & RN-50 seed 1 & RN-50 seed 2 \\
\toprule
\emph{Target Accuracy} & 94.73 & 94.84 & 95.12 & 95.49 & 95.53 & 95.64 \\
\midrule
Source & 2.07 & 2.00 & 1.69 & 1.83 & 1.82 & 1.73 \\
\midrule
CBIW (features) & 0.59 & 0.62 & 0.20 & 0.35 & 0.68 & 0.56 \\
KMM (features) & 2.21 & 2.04 & 1.65 & 1.84 & 1.76 & 1.68 \\
uLSIF (features) & 2.22 & 2.04 & 1.65 & 1.84 & 1.76 & 1.68  \\
CBIW-FT & 0.57 & 0.58 & 0.11 & 0.41 & 0.42 & 0.32 \\
\midrule
CBIW (slices) & -0.22 & -0.07 & -0.18 & 0.17 & 0.19 & 0.12 \\
KMM (slices) & -0.45 & 0.54 & -0.14 & 0.94 & 0.80 & 0.64 \\ 
uLSIF (slices) & -0.44 & 0.55 & -0.13 & 0.93 & 0.80 & 0.65 \\
Simple (slices) & -0.22 & -0.08 & -0.18 & 0.16 & 0.19 & 0.12 \\
\name\ (slices) &  -0.22 & -0.08 & -0.18 & 0.16 & 0.19 & 0.12 \\
\bottomrule
        \end{tabular}

    }
    \caption[\name{} Full CelebA Results]{Error in accuracy estimates for each CelebA model for all methods, for all model runs. True target \emph{accuracy} in first column.}
    \label{tab:celeba-full}
\end{table}

\subsubsection{\civil}
\para{Datasets}. The \civil dataset\footnote{The dataset can be downloaded e.g. by following the instructions at \url{https://github.com/p-lambda/wilds}.} \citep{borkan2019nuanced} contains comments labeled ``toxic'' or ``non-toxic'', along with 8 metadata labels on whether a particular identity (male, female, LGBTQ, etc.) is referenced in the text. (Any number of these metadata labels can be true or false for a given comment.) The original dataset has 269,038 training, 45,180 validation, and 133,782 test datapoints.

\para{Preprocessing}. We modify the test (target) set to introduce distribution shift by randomly subsampling examples for each ``slice'' (subset of data with a given assignment of metadata labels), with different proportions per slice. Specifically, for each of the $2^8$ possible slices, we pick a uniform random number from 0 to 1 and keep only that fraction of examples in the test set, discarding the rest. We do this for three different random seeds to produce three different ``shifted'' datasets.

\para{Models}. We fine-tune BERT-Base-uncased for 5 epochs on the \civil training dataset, using the implementation and hyperparameters provided by \citep{koh2020wilds}. We evaluate how well \name\ and the baseline methods estimate the accuracy of this models on the different test (target) set generated as described above.

\para{Methods}. We use the default implementation of \name\ (with no noise correction factor). We identify the top $4$ entries by magnitude in the inverse covariance matrix of the $\tilde{g}_i$'s and use these as our edge list. For CBIW, we again use the scikit-learn LogisticRegression function with the default regularization strength of 1.0 and the L-BFGS optimizer. For CBIW-FT, we fine-tune BERT-base-uncased for 1 epoch with the same hyperparameter settings on the concatenated source and target datasets, and with the label 0 for source images and 1 for target images instead of the true labels. (Surprisingly, we found that training for more epochs actually caused CBIW-FT to do worse on this dataset, in terms of the final estimation error.) For KMM and uLSIF, we use the implementation from \citep{fang2020rethinking} and randomly subsample 10,000 examples each from the source and target datasets before running due to the high computational and/or memory demands. We also compare to the slice-based baselines discussed in Appendix~\ref{app:baselines-first}. 

\noindent Trial-by-trial results for each method are given in Table~\ref{tab:civil-full}.

\begin{table}[]
    \centering
        \begin{tabular}{l|c|c|c}
            \toprule
& Dataset seed 0 & Dataset seed 1 & Dataset seed 2 \\
\toprule
\emph{Target Accuracy} & 93.42 & 89.46 & 93.40 \\
\midrule
Source & -0.92 & 3.04 & -0.90 \\
\midrule
CBIW (features) & 0.01 & -0.00 & 0.08 \\
KMM (features) & -0.11 & 3.60 & -0.03 \\
uLSIF (features) & -0.59 & -0.03 & -0.55 \\
CBIW-FT & 0.10 & -0.52 & 0.03 \\
\midrule
CBIW (slices) & -0.09 & -0.57 & 0.10 \\
KMM (slices) & 0.04 & 0.15 & -0.02 \\ 
uLSIF (slices) & 0.15 & 0.72 & 0.04 \\
Simple (slices) & 0.09 & -0.01 & 0.16 \\
\name\ (slices) & 0.05 & -0.18 & 0.13 \\
\midrule
CBIW (noisy slices) & -0.15 & -0.07 & 0.02  \\
KMM (noisy slices) & -0.12 & 0.44 & -0.05 \\ 
uLSIF (noisy slices) & 0.10 & 0.88 & 0.08 \\
Simple (noisy slices) & 0.01 & 0.35 & 0.07 \\
\name\ (noisy slices) & -0.04 & 0.22 & 0.03 \\
            \bottomrule
        \end{tabular}%
    \caption[\name{} Full \civil Results]{Error in \civil model accuracy estimates for all methods, for each of the three randomly generated target datasets.}
    \label{tab:civil-full}
\end{table}

\subsubsection{\snlimnli}

\para{Datasets.} We use the \snli and the \mnli matched validation sets, which consist of $10000$ and $9815$ examples respectively.

\para{Models Evaluated.} We consider $8$ models from the Huggingface Model Hub\footnote{\url{https://huggingface.co/models}}. We provide the model identifiers, and performance across methods in Table~\ref{tab:snlimnli-perf-by-model} and \ref{tab:snlimnli-perf-by-model-2}.

\para{Performance Metrics.} We consider estimation of $2$ performance metrics: 
\begin{enumerate}
    \item \textit{Standard Accuracy.} Average classification accuracy across all $3$ NLI classes: entailment, neutral and contradiction.
    \item \textit{Binary Accuracy.} Average classification accuracy across $2$ classes: entailment and non-entailment, where non-entailment consists of examples labeled either neutral or contradiction.
\end{enumerate}

\para{Slices.} We consider $9$ slices in total, derived from two sources of information. Suppose that ${\bf p} = f_\theta(x)$ are the class probabilities assigned by $f$ to $x$. Then,
\begin{enumerate}
    \item \textit{Model Predictions (3 slices).} We use the slice $g_j(x) = \mathbf{1}[\argmax_{i \in [3]}{\bf p} = j]$ for $j \in [3]$ i.e. all examples where the model assigns the label to be $j$.
    \item \textit{Model Entropy (6 slices).} We calculate the entropy of the output predictions $H({\bf p}) = - \sum_{i \in [3]} p_i \log p_i$. We then consider slices $g_j(x) = \mathbf{1}\left[H({\bf p}) \in [0.2(j - 1), 0.2j]\right]$ for $j \in [6]$ i.e. all examples where the model's prediction entropy lies in a certain interval.
\end{enumerate}

Intuitively, slice statistics over model predictions act as a noisy indicator for the label, and potentially capture spurious associations made by the model. Model entropy captures where the model is more or less uncertain, e.g. an increase in high entropy examples is a good indicator that the distribution has shifted significantly.

\noindent {\bf Methods.} We use the default implementation of \name\ (with no noise correction factor). We compare \name\ to \cbiw, \kmm and \cbiwft. For \cbiw, we use the LogisticRegression function from scikit-learn, with default regularization ($C=1.0$), training until convergence. For \kmm, we use the implementation provided by \citep{fang2020rethinking}, with a kernel width of $1.0$. For \ulsif, we use the implemented provided by \citep{fang2020rethinking}. For \cbiwft, we train a \texttt{bert-base-uncased} model using the Huggingface Trainer for $3$ epochs.

We also compare to applying the same baselines (aside from CBIW-FT which is not applicable) directly on the \emph{slice} representations, rather than the features, i.e. swapping out the KLIEP-based stage of \name\ for another importance weighting algorithm. We additionally compare to the ``simple" baseline of reweighting points based on the ratio of the frequency of their slice in the target dataset to the frequency of their slice in the source dataset. These results can be found in Table~\ref{tab:snlimnlifull}.

\newcommand{\downarrowgreen}{{\color{green} \downarrow}}
\newcommand{\downarrowred}{{\color{red} \downarrow}}
\newcommand{\uparrowgreen}{{\color{green} \uparrow}}
\newcommand{\uparrowred}{{\color{red} \uparrow}}
\newcommand{\approxgreen}{{\color{green} \approx}}
\newcommand{\approxred}{{\color{red} \approx}}
\begin{table}[]
    \centering
    \resizebox{\linewidth}{!}{%
        \begin{tabular}{l|c|c|c|c|c|c|c}
            \toprule
            Model  & \snli & \mnli & \name & \cbiw & \kmm & \cbiwft & \ulsif \\
            \midrule
            {\tt ynie/roberta-large-snli\_mnli\_fever\_anli\_R1\_R2\_R3-nli}
            & 91.09 
            & 89.88 $\downarrow$ 
            & 88.70 $\downarrowgreen$ 
            & 92.12 $\uparrowred$ 
            & 91.01 $\approxred$
            & 90.49 $\downarrowgreen$ 
            & 90.98 $\downarrowgreen$
            \\
            {\tt textattack/bert-base-uncased-snli}
            & 89.51 
            & 73.88 $\downarrow$ 
            & 79.94 $\downarrowgreen$ 
            & 91.76 $\uparrowred$ 
            & 88.45 $\downarrowgreen$
            & 88.70 $\downarrowgreen$ 
            & 89.93 $\uparrowred$
            \\
            {\tt facebook/bart-large-mnli}
            & 87.48 
            & 90.18 $\uparrow$ 
            & 87.28 $\approxred$ 
            & 89.30 $\uparrowgreen$ 
            & 87.89 $\uparrowgreen$
            & 86.71 $\downarrowred$ 
            & 87.43 $\approxred$
            \\
            {\tt textattack/bert-base-uncased-MNLI}
            & 78.58 
            & 84.58 $\uparrow$ 
            & 79.58 $\uparrowgreen$ 
            & 82.32 $\uparrowgreen$ 
            & 79.03 $\uparrowgreen$
            & 77.75 $\downarrowred$ 
            & 77.87 $\downarrowred$
            \\
            {\tt huggingface/distilbert-base-uncased-finetuned-mnli}
            & 74.76 
            & 82.25 $\uparrow$ 
            & 77.57 $\uparrowgreen$ 
            & 75.91 $\uparrowgreen$ 
            & 76.47 $\uparrowgreen$
            & 75.19 $\uparrowgreen$ 
            & 74.22 $\downarrowred$
            \\
            {\tt prajjwal1/albert-base-v1-mnli}
            & 72.33 
            & 80.12 $\uparrow$ 
            & 76.81 $\uparrowgreen$ 
            & 73.78 $\uparrowgreen$ 
            & 72.34 $\approxred$
            & 73.07 $\uparrowgreen$ 
            & 72.00 $\downarrowred$
            \\
            {\tt cross-encoder/nli-deberta-base}
            & 90.67 
            & 88.24 $\downarrow$ 
            & 89.21 $\downarrowgreen$ 
            & 92.55 $\uparrowred$ 
            & 90.36 $\downarrowgreen$
            & 89.64 $\downarrowgreen$ 
            & 90.89 $\uparrowred$
            \\
            {\tt squeezebert/squeezebert-mnli}
            & 76.18 
            & 82.92 $\uparrow$ 
            & 77.77 $\uparrowgreen$ 
            & 79.37 $\uparrowgreen$ 
            & 76.59 $\uparrowgreen$
            & 75.35 $\downarrowred$ 
            & 76.68 $\uparrowgreen$
            \\
            \bottomrule
        \end{tabular}%
    }
    \caption[\name{} \snlimnli Results: Standard Accuracy]{Standard accuracy estimates for $8$ models from the Huggingface Model Hub across all methods. Direction of adjustment with respect to the original \snli estimate is indicated by $\uparrow/\downarrow/\approx$. {\color{green} Green} indicates that the adjustment was done in the correct direction, while {\color{red} red} indicates the adjustment was in the wrong direction.}
    \label{tab:snlimnli-perf-by-model}
\end{table}

\begin{table}[]
    \centering
    \resizebox{\linewidth}{!}{%
        \begin{tabular}{l|c|c|c|c|c|c|c}
            \toprule
           Model  & \snli & \mnli & \name & \cbiw & \kmm & \cbiwft & \ulsif \\
            \midrule
            {\tt ynie/roberta-large-snli\_mnli\_fever\_anli\_R1\_R2\_R3-nli}    
            & 93.89 
            & 93.58 $\downarrow$ 
            & 92.32 $\downarrowgreen$
            & 93.80 $\approxred$
            & 93.67 $\downarrowgreen$
            & 94.05 $\approxred$
            & 93.45 $\downarrowgreen$
            \\
            {\tt textattack/bert-base-uncased-snli}                             
            & 93.10 
            & 83.78 $\downarrow$ 
            & 86.37 $\downarrowgreen$
            & 93.62 $\uparrowred$
            & 92.46 $\downarrowgreen$
            & 93.28 $\approxred$
            & 92.84 $\downarrowgreen$
            \\
            {\tt facebook/bart-large-mnli}                                      
            & 92.32 
            & 94.05 $\uparrow$ 
            & 92.15 $\approxred$
            & 90.88 $\downarrowred$
            & 92.22 $\approxred$
            & 91.80 $\downarrowred$
            & 91.63 $\downarrowred$
            \\
            {\tt textattack/bert-base-uncased-MNLI}                             
            & 87.75 
            & 90.77 $\uparrow$ 
            & 88.30 $\uparrowgreen$
            & 88.01 $\uparrowgreen$
            & 87.09 $\downarrowred$
            & 87.28 $\downarrowred$
            & 86.36 $\downarrowred$
            \\
            {\tt huggingface/distilbert-base-uncased-finetuned-mnli}            
            & 86.81 
            & 89.52 $\uparrow$ 
            & 88.00 $\uparrowgreen$
            & 84.57 $\downarrowred$
            & 86.85 $\approxred$
            & 87.36 $\uparrowgreen$
            & 85.11 $\downarrowred$
            \\
            {\tt prajjwal1/albert-base-v1-mnli}                                 
            & 84.64 
            & 88.22 $\uparrow$ 
            & 87.53 $\uparrowgreen$
            & 83.13 $\downarrowred$
            & 83.12 $\downarrowred$
            & 85.11 $\uparrowgreen$
            & 82.79 $\downarrowred$
            \\
            {\tt cross-encoder/nli-deberta-base}                                
            & 93.72 
            & 92.92 $\downarrow$ 
            & 92.72 $\downarrowgreen$
            & 93.76 $\approxred$
            & 93.74 $\approxred$
            & 93.74 $\approxred$
            & 93.45 $\downarrowgreen$
            \\
            {\tt squeezebert/squeezebert-mnli}                                  
            & 87.24 
            & 89.75 $\uparrow$ 
            & 88.30 $\uparrowgreen$
            & 86.63 $\downarrowred$
            & 86.48 $\downarrowred$
            & 86.70 $\downarrowred$
            & 86.37 $\downarrowred$
            \\
            \bottomrule
        \end{tabular}%
    }
    \caption[\name{} \snlimnli Results: Binary Accuracy]{Binary accuracy estimates for $8$ models from the Huggingface Model Hub across all methods. Direction of adjustment with respect to the original \snli estimate is indicated by $\uparrow/\downarrow/\approx$. {\color{green} Green} indicates that the adjustment was done in the correct direction, while {\color{red} red} indicates the adjustment was in the wrong direction.}
    \label{tab:snlimnli-perf-by-model-2}
\end{table}

\begin{table}[]
    \centering
    \resizebox{0.8\linewidth}{!}{%
        \begin{tabular}{l|c|c|c|c|c}
            \toprule
           Task  &  \name & \cbiw & \kmm & \cbiwft & \ulsif \\
            \midrule
            \celeba (ResNet-50)
            & {\bf 0.03s}
            & 32s
            & 185s
            & 856s 
            & 361s
            \\
            \civil
            & {\bf 37.6s}
            & 158s
            & 306s
            & 2350s 
            & 371s
            \\
            \snlimnli
            & {\bf 0.05s}
            & 0.25s
            & 63.25s
            & 239s
            & 13.38s
            \\
            \bottomrule
        \end{tabular}%
    }
    \caption[\name{} Runtimes]{Average total time taken by each method to evaluate a single model for the given tasks.}
    \label{tab:time-taken}
\end{table}

\begin{table}[h]
    \centering
    \begin{sc}
    \resizebox{0.8\linewidth}{!}{%
    \begin{tabular}{c|c|c|c|c}
        \toprule
         \multirow{2}{*}{Method} & \multicolumn{2}{c}{Standard Accuracy} & \multicolumn{2}{c}{Binary Accuracy}\\
          & Avg. Error & Max. Error & Avg. Error & Max. Error\\
         \midrule
        {\sc Source} & $6.2\% \pm 3.8\%$ & $15.6\%$ & $3.0\% \pm 2.3\%$ & $9.3\%$  \\
        \midrule
        \cbiw$^\dagger$ & $5.5\% \pm 4.5\%$  & $17.9\%$ & $3.7\% \pm 2.5\%$ & $9.8\%$  \\
        \kmm$^\dagger$ & $5.7\% \pm 3.6\%$ & $14.6\%$ & $3.3\% \pm 2.3\%$ & $8.7\%$  \\
        uLSIF~$^\dagger$ & $6.4\% \pm 3.9\%$ & $16.0\%$ & $3.7\% \pm 2.4\%$ & $9.1\%$ \\
        \midrule
        Simple~$^\ddagger$ & ${\bf 3.5\% \pm 1.6\%}$ & ${\bf 5.9\%}$ & ${\bf 1.6\% \pm 0.7\%}$ & ${\bf 2.7\%}$ \\
        \cbiw~$^\ddagger$ & $6.3\% \pm 3.8\%$ & $15.7\%$ & $3.1\% \pm 2.4\%$ & $9.7\%$ \\
        \kmm~$^\ddagger$ & ${\bf 3.5\% \pm 1.6\%}$ & ${\bf 5.9\%}$ & ${\bf 1.6\% \pm 0.7\%}$ & ${\bf 2.7\%}$ \\
        \ulsif~$^\ddagger$ & $6.4\% \pm 4.8\%$ & $18.9\%$ & $3.1\% \pm 3.4\%$ & $11.6\%$ \\
        \name~$^\ddagger$ & ${\bf 3.6\% \pm 1.6\%}$ & ${\bf 5.9\%}$ & ${\bf 1.6\% \pm 0.7\%}$ & ${\bf 2.7\%}$  \\
        \bottomrule
        \multicolumn{5}{r}{$\dagger$ use raw features, $\ddagger$ use slices}
    \end{tabular}%
    }
    \end{sc}
    \caption[\name\ \snlimnli Accuracy Estimates: Baseline Evaluation on Slice Representations]{Estimating standard and binary accuracy on {\sc MNLI} using {\sc SNLI}. Average and maximum estimation errors for $8$ models are reported, with $95\%$ confidence intervals. This table extends Table~\ref{tab:snlimnli} to include ablations where baseline methods were run directly on slices as well. This validates that the slice-based representation provides a significant advantage over traditional feature-based approaches.}
    \label{tab:snlimnlifull}
\end{table}

\subsubsection{\snliphansm}

\para{Datasets.} We use the \snli and the \hans validation sets, which consist of $10000$ and $30000$ examples respectively. We further process them to construct the \snlip and \hansm datasets as detailed next.

\para{Preprocessing.} We randomly sample $1\%$ ($300$ examples) of the \hans data and add it to \snli to construct the \snlip dataset. The remaining $99\%$ of the \hans dataset constitutes the \hansm dataset.

\para{Models Evaluated.} We consider the same $8$ models as \snlimnli. 

\para{Slices.} We add slices based on contradiction (based on negation and token ordering), sentence structure (word substitutions, length differences, verb tense) and the evaluated model's uncertainty (similar to \snlimnli).


\subsubsection{IMDB}
{\bf Datasets.} For the source dataset, we use the test split of the IMDB sentiment classification dataset~\citep{maas2011learning}. For the target datasets, we use (i) Counterfactual IMDB test dataset~\citep{kaushik2020learning}; (ii) Sentiment 140 test dataset~\citep{Sentiment140}; (iii) the first $2000$ examples from the Yelp Polarity reviews test dataset~\citep{zhangCharacterlevelConvolutionalNetworks2015}, taken from Huggingface datasets; (iv) the first $2000$ examples from the Amazon Polarity reviews test dataset~\citep{zhangCharacterlevelConvolutionalNetworks2015}, taken from Huggingface datasets.

\noindent {\bf Models Evaluated.} We consider $3$ models taken from the Huggingface Model Hub: \\ {\tt textattack/bert-base-uncased-imdb}, {\tt textattack/roberta-base-imdb}, and \\ {\tt textattack/distilbert-base-uncased-imdb}.

\noindent  {\bf Slices.} We use the same task-agnostic slices as \snlimnli, with model predictions and model entropy.



\subsection{Runtime}
We include runtime information in Table~\ref{tab:time-taken}. For \snlimnli, we run all experiments on a \texttt{g4dn.2xlarge} machine on AWS. For \civil, we run all experiments on a \texttt{p3.2xlarge} on AWS, and for \celeba we run on a \texttt{p3.8xlarge} on AWS. Both \cbiw and \name\ run quite fast, while the other methods take significantly more time---\cbiwft fine-tunes a neural network rather than training a simple logistic classifier, and KMM is very compute- and memory-intensive (even after we downsample to 10,000 points as in the case of \civil and \celeba). Unsurprisingly, \name~runs especially fast when there are fewer slicing functions (for instance on \celeba).

\subsection{Baseline Analysis} \label{app:baselines}

While the primary contribution of \name\ is to correct distribution shift based on slices, our experiments have also highlighted the effects of using different methods for computing the density ratio on the slices. Comparisons among \cbiw, \kmm, \ulsif, and \llkliep have been done before, as discussed below:
\begin{itemize}
    \item \cbiw using logistic regression yields weights with the same parametric form as \llkliep (i.e. log-linear)~\citep{Tsuboi2009}. \cbiw has lower asymptotic variance when the distribution belongs to the exponential family but does worse than \llkliep when the exponential family is misspecified~\citep{kanamori2010}.
    \item \kmm is relatively slow (as suggested by Table \ref{tab:time-taken}) and needs quite a bit of fine-tuning on the kernel and regularization parameters. In the setting we study in which the target dataset labels are unknown, this fine-tuning cannot be done with cross-validation.
    \item \ulsif differs from \llkliep in that it uses the squared loss rather than the log loss and thus can be more efficient~\citep{JMLR:v10:kanamori09a}. However, \llkliep's loss has a more natural interpretation for exponential distributions over the binary slices and allows us to explicitly address correlations and noise among the slices.
\end{itemize}

Lastly, perhaps the most intuitively obvious way to use slice information to reweight accuracy estimates is to simply reweight each point by the ratio of the frequency of its slice vector in the target dataset to the frequency of its slice vector in the source dataset. (For example, if there are three slices $g_1,g_2,g_3$, the source dataset and target dataset both have 100 total examples, the source dataset has 10 examples with $g_1 = 1, g_2 = 0, g_3 = 1$, and the target dataset has 5 such examples, we would reweight all these examples by 1/2.) This is the ``simple" baseline mentioned in Appendix~\ref{app:baselines-first}. While this method is extremely simple to run, it can introduce a substantial amount of noise due to not exploiting any possible structure of the relationships between the slices; with $k$ binary slices, there can be $2^k$ unique slice vectors, and one parameter must be estimated for each one that is observed in the source dataset.

By contrast, in our graphical model, many fewer parameters may need estimation. For example, suppose that the slices are actually independent. Suppose for simplicity that our source dataset is arbitrarily large (to isolate the effect of sampling noise in the target dataset), and that the slices are noiseless. Suppose that there are $k$ individual slicing functions, where each is an independent Bernoulli random variable, with probability $1/2$ on the source dataset and $p$ on the target dataset (with $p \in (0, 1)$ unknown). Additionally, suppose that a datapoint is ``correct" exactly when all of the slicing functions evaluate to 1. In this (contrived) setting, we could simply estimate the accuracy on the target \emph{distribution} by counting the number of datapoints in the target set for which all slicing functions evaluate to 1 (which is equivalent to the accuracy on the target dataset), and dividing by the number of target datapoints $n_t$. This is a random variable with mean $p^k$ (the true target accuracy) and variance $p^k(1-p^k) / n_t$. However, the slice information can actually help us obtain a \emph{better} estimate than this. For instance, we could instead estimate the empirical mean of each slicing function on the target set independently; these would be RVs with mean $p$ and variance $p(1-p)/n_t$. Then, we could compute the product as our estimate of the probability of all slices being 1; by independence, the mean of the result is $p^k$ and the variance is $\left( \frac{p(1-p)}{n_t}+p^2 \right)^k - p^{2k}$. (It can be shown that in this specific case, \name~with an empty edgeset $E = \emptyset$ is equivalent to the latter approach in the $n_s \ra \infty$ limit.) When $k$ is large, the variance of the naive approach is $\approx p^k / n_t$, while the variance of the latter approach is smaller: $\approx p^{2k-1}(1-p)/n_t$ when $n_t$ is also large. This informal argument highlights that, as the number of slices grows, explicit knowledge of the dependence structure of the slicing functions (in this case, full independence) can significantly improve the quality of estimates.
\section{Extended Related Work}
\label{app:ext-related}
\subsection{Correcting for Bias in Observational Studies} Studies in many fields suffer from selection bias when the subjects are not representative of the target population. For instance, polling is biased due to sample demographics not aligning with those of the true population~\citep{isakov2020towards}. In other disciplines ranging from public policy to health services, investigators cannot control which individuals are selected to receive a treatment in observational studies unlike in randomized controlled trials. Therefore, there may be a significant difference in the treatment outcomes of individuals that were selected and those that were not. \textit{Propensity scores}, the probability of an individual being assigned treatment given their characteristics, are a standard way to correct for this bias in the causal inference literature. Two common methods are \emph{matching}, where each treated subject is matched with an untreated subject based on similarity of their propensity scores~\citep{rosenbaummatching}, and \emph{weighting}, where estimators of the treatment effect are weighted by functions of the propensity score~\citep{HiranoImbens2001}. The latter method is technically similar to IW; in fact, the weights used for estimating the average treatment effect for the control (ATC) are proportional to the density ratio used in importance weighting if we consider the control group as the target distribution~\citep{LiMorganZaslavsky2018}. However, a key difference between propensity score weighting and \name\ is that in observational studies estimates are reweighted using a small known set of demographics of subjects, whereas in \name\ we first use slicing functions to identify these relevant covariates.

\subsection{Distribution Shift} Distribution shift has been addressed in various problem settings. Domain adaptation tackles how to train a model to handle distribution shift, which usually involves unsupervised methods that perform additional training using unlabeled target data. In particular, the loss function is reweighted (such as via importance weighting!)~\citep{pmlr-v97-byrd19a}, or includes some additional terms or constraints to capture discrepancy between the source and target~\citep{long16, JMLR:v20:15-192}. Distributionally Robust Optimization (DRO) tackles a similar problem involving distribution shift, where a model is trained to optimize over and be robust over an uncertainty set of distributions rather than a single specified target distribution~\citep{Ben-Tal2013, duchi2018statistics}. In contrast, our goal is only to evaluate existing models, regardless of how they were trained. We assume the model to be a fixed ``black box'' and estimate its performance using unlabeled target data. A key point of our approach is the ability to cheaply compare arbitrary models on one or more target distributions, without requiring any re-training.

Several other works have focused on structured distribution shift in ways similar to our setup in Section \ref{subsec:model-dist-shift}. A latent structure consisting of shifting and non-shifting components has been examined in DRO~\citep{subbaswamy2020evaluating, pmlr-v80-hu18a}, but the structure does not consider relevance of the components to the task. On the other hand, the dimensionality reduction approach in~\cite{SUGIYAMA201044} has been extended to consider relevance of features to $Y$~\citep{pmlr-v89-stojanov19a}, but this lacks the interpretability of binary slices. More recently, \cite{polo2020covariate} propose a task-aware feature selection approach to determine inputs to importance weighting. \name\ describes a structure on the distributions that addresses both shifting versus non-shifting and irrelevant versus relevant components. Furthermore, a subtle but important distinction is that our categorization is on user-defined slices rather than features or known latent variables, which naturally suggests slice design as an iterative process in the evaluation framework whereas other approaches consider static input to extract properties from.  


While our work focuses on the \emph{covariate shift} setting, another commonly studied form of distribution shift is \emph{label shift}, which assumes that $p_s(x | y) = p_t(x | y)$ but $p_s(y) \neq p_t(y)$. Correcting for label shift requires density ratio estimation approaches different from standard IW~\citep{lipton2018detecting} and is an interesting setting for future work in model evaluation.

\subsection{Density Ratio Estimation} In addition to the methods discussed in Section \ref{sec:IW}, recent approaches have focused on reducing the variance of the importance weights. \cite{LiRobust2020} combine KMM with nonparametric regression to reduce variance of the former and bias of the latter, and \cite{Rhodes2020} estimate a series of lower-variance density ratios, whose product telescopes into the desired density ratio. Another method is to discard outliers in the datasets to improve the boundedness of the weights~\citep{TrimmedEstimation2017}. While these methods produce more accurate estimates on target distributions, they do not do so by directly addressing the cause of high variance weights, which \name{} does by only reweighting on relevant shifting properties of the data.  

\subsection{Active Evaluation} Inspired by active learning~\citep{Settles2012Active}, active evaluation methods tackle the evaluation problem by creating a small, labeled test set as a proxy for testing the model's performance on the fully labeled target distribution~\citep{Nguyen2018ActiveTest, Rahman2020EfficientTest}. These methods use a model's predictions on the target distribution to guide the selection of examples to be manually labeled, reducing the cost to construct a test set on the target distribution. Several active evaluation methods have been proposed; for example, \cite{Taskazan2020GrandfatherTestSet} apply active evaluation to the domain shift problem by stratifying the source distribution, weighting the inferred strata according to the target distribution, and then sampling examples to label using the inferred weights. Instead of requiring costly annotation of individual examples from the target distribution, \name\ lets the user provide high-level guidance in the form of \textit{slicing functions}, which automatically label the dataset with noisy labels, and then estimates performance on the target set automatically.

\subsection{Other Applications of Importance Weighting \& Importance Sampling}
While \textsc{Mandoline} focuses on the use and adaptation of importance weighting for evaluation of classification models, importance weighting and the closely related concept of importance sampling have a rich history of application in many domains in addition to those mentioned above. Importance sampling is a crucial component of efficient simulation in many domains, used to reduce variance and thereby enable computational savings \citep{192731}. Similarly, importance weighting is a key technique in reinforcement learning, for off-policy value function estimation \citep{suttonbarto}: given episodes from a behavior policy, observed returns can be reweighted by the importance ratio to estimate expected rewards following a different \emph{target} policy.

\end{document}